\documentclass{article}
\usepackage{iclr2027_conference,times}

\usepackage[utf8]{inputenc} 
\usepackage[T1]{fontenc}    
\usepackage{hyperref}       
\usepackage{booktabs}       
\usepackage{amsfonts}       
\usepackage{nicefrac}       
\usepackage{microtype}      
\usepackage{xcolor}         
\hypersetup{hidelinks}

\usepackage{soul}
\usepackage{url}
\usepackage{xspace}
\usepackage{enumitem}
\usepackage[small]{caption}
\usepackage{graphicx}
\usepackage{flafter}
\usepackage{amsmath}
\usepackage{amsthm}
\usepackage{algorithm}
\usepackage{algorithmic}
\usepackage{multirow}
\usepackage{subcaption}  
\usepackage{tikz}
\usepackage{textcomp}
\usepackage{bm}
\usepackage{colortbl}
\definecolor{COLOR_MEAN}{RGB}{240,240,240}  
\usepackage{amsthm}
\newtheorem{proposition}{Proposition}

\let\oldhat\hat

\renewcommand{\hat}[1]{\oldhat{\mathbf{#1}}}

\setlist[itemize]{leftmargin=1.3em,labelsep=0.45em,topsep=0.2em,itemsep=0.15em,parsep=0pt,partopsep=0pt}
\title{Topology-Consistent Task Planning over {\color{black}\mbox{Cellular}} Workflow
Complexes for LLM-based Agents}

\renewcommand{\thefootnote}{\fnsymbol{footnote}}
\author{
\textbf{Sen Zhao\textsuperscript{1} \quad Jia Tang\textsuperscript{1} \quad Ruiqi Kong\textsuperscript{1}} \\
\textbf{Zuyu Zhang\textsuperscript{1} \quad Lifeng Shen\textsuperscript{2} \quad Ding Zou\textsuperscript{3}} \\
\textbf{Xinyu He\textsuperscript{4} \quad Xu Zhang\textsuperscript{1}\thanks{Corresponding authors: Xu Zhang and Junwei Han.} \quad Junwei Han\textsuperscript{5,2}\footnotemark[1]} \\[0.35em]
{\normalfont\small \textsuperscript{1}Academy of Advanced Interdisciplinary Studies,} \\
{\normalfont\small Chongqing University of Posts and Telecommunications, Chongqing, China} \\
{\normalfont\small \textsuperscript{2}School of Artificial Intelligence,} \\
{\normalfont\small Chongqing University of Posts and Telecommunications, Chongqing, China} \\
{\normalfont\small \textsuperscript{3}Intelligent System Department, Zhongxing Telecom Equipment (ZTE),} \\
{\normalfont\small Changsha, Hunan, China} \\
{\normalfont\small \textsuperscript{4}Towngas, China} \\
{\normalfont\small \textsuperscript{5}School of Automation, Northwestern Polytechnical University, Xi'an, China}
}

\iclrfinalcopy

\begin{document}

\maketitle
\fancyhead{}
\renewcommand{\thefootnote}{\arabic{footnote}}
\setcounter{footnote}{0}

\begin{abstract}
  Task planning for LLM agents requires workflows that satisfy both user intent and complex sub-task dependencies. While existing planners work well for sequential or directed acyclic graph (DAG)-like structures, they struggle with workflow patterns such as verification-correction loops, convergent branch merging, and reusable intermediate states that arise naturally in real-world tool orchestration. We present TopoPlanner, a topology-consistent planning framework that lifts tool dependency graphs into cellular workflow complexes and uses them as topology-aware context for LLM tool planning. TopoPlanner retrieves a request-relevant closed subcomplex through cosheaf-consistent cellular retrieval, performs multi-dimensional structural reasoning over the retrieved topology, and interfaces the resulting cellular representation with the planner LLM for tool-sequence generation. {\color{black}Experiments on four tool-planning benchmarks with {\color{black}topology-guided} loop, merge, and loop-merge workflows show consistent improvements over prompt-based and graph-enhanced baselines across different local LLM backbones.}
\end{abstract}

\section{Introduction}
Large language model (LLM) agents are becoming a practical interface for tool use, autonomous exploration, and complex decision making~(\cite{wang2024survey,bubeck2023sparks,DBLP:conf/acl/JiangZZS0ZW25}). Their effectiveness depends critically on task planning: given a natural-language request, the agent must decompose the request into executable sub-tasks, organize their dependencies, and produce a workflow that can be grounded into tools or actions~(\cite{schick2024toolformer,shen2024hugginggpt,shen2024taskbench}). In realistic agentic settings, this workflow must capture not only individual sub-tasks, but also the dependency structure that makes them executable as a coherent process.

This exposes a structural gap in existing planners. Most LLM planning methods are strongest when the workflow is sequential or approximately directed acyclic graph (DAG)-like, where pairwise task adjacency already captures most of the required structure~(\cite{wei2022chain,yao2024tree,shinn2024reflexion}). As illustrated in Fig.~\ref{fig:motivation}, real-world tool orchestration is often more complex: it may involve verification-correction loops, convergent branch merging, and reusable intermediate states. {\color{black}Although these structures can be encoded by edges, their joint regional organization remains implicit in pairwise message passing and pooled graph summaries.} In these settings, the key dependency patterns arise from how multiple relations jointly form loop, merge, and reuse structures, rather than from isolated edges. {\color{black}This motivates \emph{topology-consistent task planning} beyond pairwise adjacency.}
\begin{figure}[!t]
    \centering\includegraphics[width=1\columnwidth]{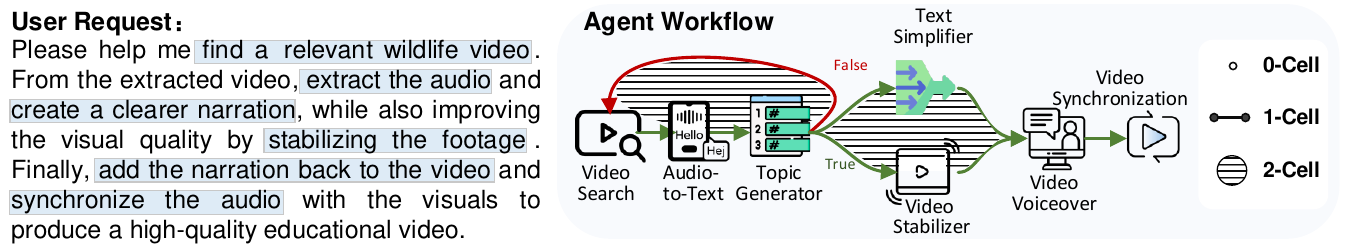}
    \caption{{\color{black}LLM task planning across workflow topologies.}}
    \label{fig:motivation}
    \vspace{-0.5em}
\end{figure}

Recent graph-enhanced approaches inject structural bias into agent planning by coupling LLMs with graphs or GNNs~(\cite{tang2024graphgpt,chai2023graphllm,DBLP:conf/nips/WuSSSWZFCCXL24,DBLP:journals/ral/ChenFCLL25,chen2026gtool}). GTool~(\cite{chen2026gtool}) studies graph-enhanced tool planning by encoding tool-dependency graphs with a conventional node-edge GNN and conditioning the planner on one pooled graph token. These methods improve over pure prompting, but they still reason mainly over node-edge structure and therefore leave cyclic or convergent workflow regions implicit. This is restrictive when the critical structure of a workflow lies in loop, merge, and reuse patterns rather than in isolated pairwise relations. Meanwhile, recent evidence suggests that cyclic structure is not incidental: stronger reasoning models tend to exhibit richer loop patterns in their reasoning graphs~(\cite{DBLP:journals/corr/abs-2506-05744}).

This perspective motivates a unified algebraic-topological formulation of workflow structures. Starting from a task workflow graph, we promote independent cycles induced by non-tree edges into explicit cycle-level cells, resulting in a \emph{cellular workflow complex}. {\color{black}In this complex, boundary relations organize sub-tasks, dependencies, and non-linear workflow regions within one multi-dimensional structure, making region-level dependencies explicit rather than implicit in pairwise paths.}

Building on this formulation, we propose \textbf{TopoPlanner}, a topology-consistent task planning framework for LLM agents. TopoPlanner lifts the tool dependency graph into a cellular workflow complex, retrieves a request-relevant closed subcomplex through cosheaf-consistent cellular retrieval, performs multi-dimensional structural reasoning over the retrieved topology, and interfaces the resulting cellular representation with the planner LLM for tool-sequence generation. To evaluate this setting, we extend four tool-planning datasets, including the three TaskBench domains~(\cite{shen2024taskbench}) and ToolBench~(\cite{chen2026gtool}), with {\color{black}topology-guided} {\color{black}loop, merge, and loop-merge cases} that emphasize non-linear workflow structure. Experiments show that TopoPlanner consistently improves planning quality over prior baselines. Our contributions in this work are summarized as follows:
\begin{itemize}
    \item {\color{black}We formulate loop, merge, and reuse dependencies as explicit regions in cellular workflow complexes, providing a unified multi-dimensional representation of non-linear workflows.}

    \item {\color{black}We introduce TopoPlanner, which lifts a workflow graph into a cellular complex used for cosheaf-consistent retrieval, multi-dimensional reasoning, and topology-guided planning.}

    \item {\color{black}We construct topology-guided extensions of four benchmarks with loop, merge, and loop-merge cases, and show consistent improvements over the compared baselines across different local LLM backbones.}
\end{itemize}

\section{Preliminaries}
\noindent\textit{\textbf{Definition 1. (Cell Complex).}}~\cite{HGh19}
A \textbf{regular cell complex} is a topological space $X$ decomposed into a collection of disjoint subspaces $\{X_\sigma\}_{\sigma \in P_X}$, referred to as \textit{cells}, satisfying the following conditions:
\begin{enumerate}
    \item For each $x \in X$, there exists an open neighborhood intersecting only finitely many cells.
    \item For any pair of cells $X_\sigma, X_\tau$, the intersection $X_\tau \cap \overline{X_\sigma}$ is nonempty if and only if $X_\tau \subseteq \overline{X_\sigma}$, where $\overline{X_\sigma}$ denotes the topological closure of $X_\sigma$.
    \item Each cell $X_\sigma$ is homeomorphic to an open ball in $\mathbb{R}^n$ for some non-negative integer $n$.
    \item \textit{(Regularity)} The closure $\overline{X_\sigma}$ of every cell is homeomorphic to a closed ball in $\mathbb{R}^{n_\sigma}$, with the interior mapped homeomorphically onto $X_\sigma$ itself.
\end{enumerate}
{\color{black}\noindent\textit{\textbf{Definition 2.}} Given a directed workflow graph \(G=(V,E,T)\), let \(G_u\) denote its underlying undirected graph, with the original edge directions retained as attributes. For a connected \(G_u\), a chosen spanning tree \(\mathcal{T}\) defines a \textbf{cellular lifting map} \(f_{\mathcal{T}}\) that constructs a regular cell complex \(X=f_{\mathcal{T}}(G)\): its 0-cells and 1-cells correspond to \(V\) and \(E\), and its 2-cells are attached along the fundamental cycles of \(G_u\) induced by non-tree edges. Thus, the original directed dependencies remain on the 1-skeleton while cycle-level workflow regions are made explicit.}
\par\smallskip
\noindent\textit{\textbf{Definition 3. (Topology-Consistent Task Planning).}}
{\color{black}Let $G = (V, E, T)$ be a text-attributed task graph, and let $X = f_{\mathcal{T}}(G)$ be its lifted regular cell complex under a cellular lifting map $f_{\mathcal{T}}$ (Definition 2), where cells encode sub-tasks and their higher-order dependency structures.}
  {\color{black}Given a user instruction, \emph{topology-consistent task planning} selects a request-relevant subcomplex $\widehat{X} \subseteq X$ according to two structural criteria: (i) boundary completeness, i.e., for every cell $\sigma \in \widehat{X}$, its boundary satisfies $\partial \sigma \subseteq \widehat{X}$; and (ii) request-conditioned incidence compatibility between selected cofaces and their boundary cells, measured by the residual in Section~\ref{sec:cellular_retrieval}. Retrieval favors lower incidence inconsistency rather than imposing a universal correctness threshold. These criteria define topology consistency of the retrieved structural context: they keep a selected higher-dimensional cell together with its boundary and favor incidence-compatible assignments. They do not, by themselves, constitute a correctness or executability certificate for the tool sequence generated by the planner LLM.}


\begin{figure*}[!t]
    \centering\includegraphics[width=1\textwidth]{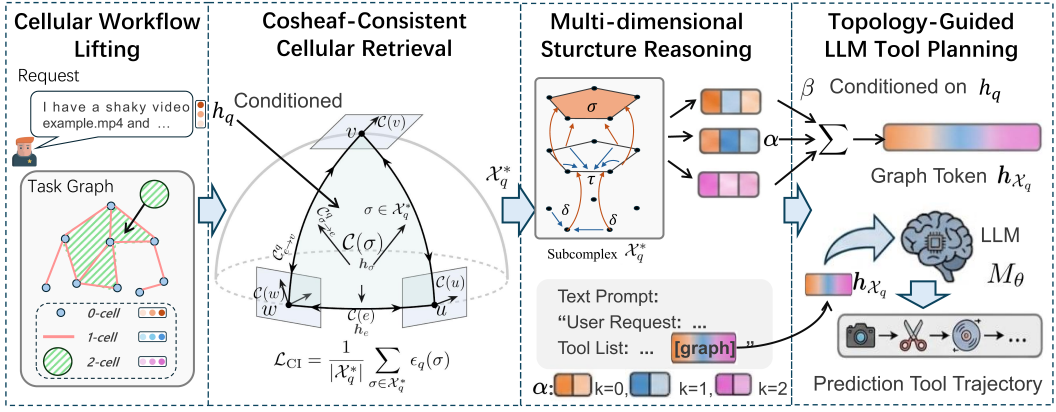}
    \caption{{\color{black}Overview of TopoPlanner.}}
    \label{fig:model}
    \vspace{-1em}
  \end{figure*}
\section{The Proposed Model}
In this section, we present \textbf{TopoPlanner} as a topology-consistent task planning framework over cellular workflow complexes, organized into four modules: \textit{Cellular Workflow Lifting}, \textit{Cosheaf-Consistent Cellular Retrieval}, \textit{Multi-Dimensional Structural Reasoning}, and \textit{Topology-Guided LLM Tool Planning}. Figure~\ref{fig:model} summarizes the overall framework. The retrieved cellular representation is interfaced with the planner LLM through a soft graph token.
\subsection{Cellular Workflow Lifting}
{\color{black}TopoPlanner views the workflow graph as the skeleton of a regular cellular complex and promotes selected cycles into explicit workflow regions. The resulting complex unifies sub-tasks, dependencies, and cycle-level structure, so all later modules operate on one multi-dimensional object.}

{\color{black}Given a directed workflow graph \(G=(V,E,T)\), we form its underlying undirected graph \(G_u\) while retaining each dependency direction as an edge attribute. We regard \(G_u\) as the 1-skeleton of a regular cell complex: each vertex is a \(0\)-cell, and each edge is a \(1\)-cell attached to its endpoints. For a connected graph, we fix a spanning tree \(\mathcal{T}\subseteq G_u\) (Appendix~\ref{algorithm}) and use the non-tree edges to identify fundamental cycles.}
\begin{equation}
{\color{black}\gamma : G_u \to G_u / \mathcal{T},}
\end{equation}
{\color{black}where \(G_u / \mathcal{T}\) denotes the quotient space obtained by collapsing \(\mathcal{T}\) into a single point. Each non-tree edge $e=(u,v)\in E\setminus\mathcal{T}$ induces a unique fundamental cycle \(\ell_e\) in \(G_u\), formed by $e$ and the tree path connecting its endpoints. We attach a cycle-level cell along each such loop:}

\begin{equation}
{\color{black}\varphi_e : \partial D^2 \to G_u^{(1)}, \qquad \varphi_e(\partial D^2) = \ell_e,}
\end{equation}
where $G^{(1)}$ is the 1-skeleton of the induced CW complex. This lifting step turns cycle-level workflow structure into explicit cells while keeping it attached to the same node-edge scaffold.


Formally, the collection of cycle-level cells constitutes
\begin{equation}
X^{(2)} = \{ x^2_e \simeq D^2 \mid e = (u,v) \in E \setminus \mathcal{T} \},
\end{equation}
{\color{black}where each disk \(D^2\) is attached along its boundary \(\partial D^2\cong S^1\) via the loop \(\ell_e\subset G_u\).}

This lifting operation completes the construction of a 2-dimensional regular cell complex:
\begin{equation}
    X = X^{(0)} \cup X^{(1)} \cup X^{(2)},
\end{equation}
where the added cycle-level cells complement the original node-edge scaffold by encoding workflow regions within the same cellular complex. 
In the following, we use \(\mathcal{X}\) to denote this lifted cellular workflow complex.
{\color{black}We encode node-cell tool descriptions with a pre-trained text encoder, initialize dependency cells from textualized relations, and derive cycle-cell features from their boundaries.}

\begin{proposition}
\label{prop.homotopy}
(Proof in Appendix \ref{app.proof.homotopy}.) 
{\color{black}$G_u/\mathcal{T}$ is homotopy-equivalent to $G_u$, and $\gamma$ induces an isomorphism on the first homology group $H_1(G_u;\mathbb{Z})$.}
\end{proposition}

\subsection{Cosheaf-Consistent Cellular Retrieval}
\label{sec:cellular_retrieval}
After lifting the workflow graph into a cellular complex, TopoPlanner retrieves a compact subcomplex that is relevant to the current request while remaining compatible with the cellular topology. Given a request \(q\), we encode it as
\begin{equation}
\bm{h}_q = \operatorname{LM}(q) \in \mathbb{R}^{d}.
\end{equation}
Rather than appending a request node and independently ranking cells of different dimensions, we view the request as a relevance functional over the lifted cellular workflow complex. Retrieval is then defined as a closed-subcomplex selection problem, while cross-dimensional compatibility is regularized through a cellular cosheaf.

\textit{Cellular cosheaf structure.}
{\color{black}We introduce a request-conditioned cellular cosheaf \(\mathcal{C}^{q}\) over the lifted complex \(\mathcal{X}\), following recent sheaf- and cosheaf-based formulations of local consistency on structured domains~(\cite{HGh19,barbero2022sheaf,wei2025persistent}). We use the reversed face-poset convention: each cell \(\sigma\) is assigned a stalk vector space \(\mathcal{C}(\sigma)\), and an incidence \(\tau\prec\sigma\) induces a coface-to-boundary corestriction map:}
\begin{equation}
\mathcal{C}^{q}_{\sigma\rightarrow\tau}:
\mathcal{C}(\sigma)\rightarrow \mathcal{C}(\tau),
\end{equation}
{\color{black}This follows the homological boundary direction; under the opposite face-poset convention, the same maps are covariant. We implement each map as a \(q\)-conditioned sigmoid gate over the request, coface, and boundary-cell embeddings.}

\textit{Unified subcomplex retrieval.}
For each cell \(\sigma\), its request relevance is computed by projecting its cellular embedding into the request embedding space:
\begin{equation}
r_q(\sigma)
=
\operatorname{cos}
\left(
\bm{h}_q,\,
\psi(\bm{h}_{\sigma})
\right),
\end{equation}
where \(\psi(\cdot)\) maps cellular embeddings into the request embedding space. For each cell \(\sigma\), we define its cosine-based cosheaf inconsistency as the average residual over its boundary cells:
\begin{equation}
{\color{black}
\epsilon_q(\sigma)
=
\begin{cases}
\displaystyle \frac{1}{|\partial\sigma|}
\sum_{\tau\prec\sigma}
\left(1-
\operatorname{cos}\left(
\mathcal{C}^{q}_{\sigma\rightarrow\tau}(\bm{h}_{\sigma}),
\bm{h}_{\tau}
\right)\right), & |\partial\sigma|>0,\\[2pt]
0, & |\partial\sigma|=0,
\end{cases}}
\end{equation}
which measures how consistently a selected coface agrees with its boundary in the retrieved complex. We retrieve a compact request-relevant subcomplex by solving:
\begin{equation}
\mathcal{X}^{*}_q
=
\underset{
\substack{
\emptyset\ne\mathcal{Y}\subseteq\mathcal{X},\\
\partial\sigma\subseteq\mathcal{Y}\ \forall\sigma\in\mathcal{Y}
}
}{\operatorname{argmax}}
\frac{1}{|\mathcal{Y}|}
\sum_{\sigma\in\mathcal{Y}}
\left(
r_q(\sigma)
\,-\,
\epsilon_q(\sigma)
\right),
\end{equation}
The selected set is kept boundary-closed, i.e., for any retained cell \(\sigma\in\mathcal{X}^{*}_q\), its boundary satisfies \(\partial\sigma\subseteq\mathcal{X}^{*}_q\). Thus, the retrieval objective favors request-relevant cells while discouraging cosheaf-inconsistent boundary assignments. 
{\color{black}Appendix~\ref{app:cellular_theory} provides further interpretations of this construction in terms of fixed-depth regional communication, incidence-consistency energy, and the role of closure in preventing consistency by omission.}
In practice, we cast the selection as a Prize-Collecting Steiner Tree (PCST)-style problem over cells~\citep{he2024gretriever}, where \(r_q(\sigma)\) is the cell prize and \(\epsilon_q(\sigma)\) is the cosheaf-induced cost. We then greedily collect high-prize cells with boundary closure. The training regularizer over the retrieved subcomplex is:
\begin{equation}
\mathcal{L}_{\mathrm{CI}}
=
\frac{1}{|\mathcal{X}^{*}_q|}
\sum_{\sigma\in\mathcal{X}^{*}_q}
\epsilon_q(\sigma).
\end{equation}

\subsection{Multi-Dimensional Structural Reasoning} 

Given the retrieved request-specific subcomplex \(\mathcal{X}^{*}_q\), TopoPlanner performs multi-dimensional reasoning over its cellular topology. This stage jointly updates representations across dimensions via boundary relations and co-facial interactions, so that downstream planning can use embeddings that reflect both local semantics and global workflow structure.

For each cell $\sigma$, the messages are defined as:
\begin{align}
m_{\mathcal{B}}^{(t+1)}(\sigma) = \operatorname{AGG}_{\tau \in \mathcal{B}(\sigma)} M_{\mathcal{B}}\left(\bm{h}_\sigma^{(t)}, \bm{h}_\tau^{(t)}\right),\\
m_{\uparrow}^{(t+1)}(\sigma) = \operatorname{AGG}_{\substack{\tau \in \mathcal{N}_\uparrow(\sigma) \\ \delta \in \mathcal{C}(\sigma, \tau)}} M_{\uparrow}\left(\bm{h}_\sigma^{(t)}, \bm{h}_\tau^{(t)}, {\color{black}\bm{h}_\delta^{(t)}}\right),
\end{align}
where \( \mathcal{B}(\sigma)=\partial\sigma \) denotes the set of lower-dimensional cells on the boundary of cell \( \sigma \), \( \mathcal{N}_\uparrow(\sigma) \) denotes the set of same-dimensional cells that co-occur with \( \sigma \) in at least one higher-dimensional cell \( \delta \), and \( \mathcal{C}(\sigma, \tau) \) denotes the set of such shared higher-dimensional cells. 

Finally, both types of structure-aware messages are combined with the current representation of the cell and passed through an update function to obtain the new representation:
\begin{equation}
{\color{black}\bm{h}_\sigma^{(t+1)} = \mathcal{U} \left( \bm{h}_\sigma^{(t)},\ m_{\mathcal{B}}^{(t+1)}(\sigma),\ m_{\uparrow}^{(t+1)}(\sigma) \right),}
\end{equation}
where \( \mathcal{U}(\cdot) \) is implemented as a nonlinear update function (e.g., ELU + linear). {\color{black}This stage lets each cell encode its local semantics and role in the multi-dimensional workflow.}

\subsection{Topology-Guided LLM Tool Planning}
After multi-dimensional structural reasoning, the retrieved subcomplex \(\mathcal{X}^{*}_q\) contains topology-aware representations of the selected cellular structure. For local planner LLMs, we summarize it as a continuous graph token and inject it together with the textual prompt.

Let \(\bm{h}^{(k)}_\sigma\) denote the updated representation of a \(k\)-cell \(\sigma\in\mathcal{X}^{*(k)}_q\). TopoPlanner summarizes the retrieved cellular structure by request-aware pooling over both cells and dimensions:
\begin{equation}
\bm{h}_{\mathcal{X}_q}
=
\sum_{k=0}^{2}
\beta_k
\sum_{\sigma \in \mathcal{X}_q^{*(k)}}
\alpha_\sigma^{(k)}
\bm{h}_\sigma^{(k)},
\end{equation}
 {\color{black}where \(\alpha_\sigma^{(k)}\) weighs cells per dimension and \(\beta_k\) balances dimensions, both conditioned on the request embedding \(\bm{h}_q\). The vector \(\bm{h}_{\mathcal{X}_q}\) encodes retrieved tools, edges, and higher-dimensional structure.}

We then construct the planner prompt from the tool list and user request, with a special graph placeholder \texttt{[graph]}...\texttt{[/graph]}. As in soft graph-token conditioning~(\cite{tang2024graphgpt,perozzi2024graphtoken}), the textual tokens are embedded by the LLM, while the placeholder embedding is replaced by \(\bm{h}_{\mathcal{X}_q}\). The planner LLM \(M_\theta\) then generates the tool trajectory:
\begin{equation}
\hat{\tau}
=
M_\theta(w_q),
\end{equation}
where \(\hat{\tau}\) is the predicted tool sequence. The output interface remains identical to standard tool-planning benchmarks: TopoPlanner generates tool names, and workflow edges are constructed and evaluated by the same downstream protocol used in the main experiments.

\paragraph{Training Objective.} 
TopoPlanner is trained to generate the target tool trajectory while reducing cosheaf inconsistency over the retrieved cellular subcomplex. Given the ground-truth tool trajectory \(\tau\), {\color{black}the tool-planning negative log-likelihood loss is}
\begin{equation}
\mathcal{L}_{\mathrm{TL}}
=
-\log p_{M_\theta}(\tau \mid w_q).
\end{equation}
We combine this tool-planning loss with the cosheaf inconsistency loss introduced above:
\begin{equation}
\mathcal{L}
=
\mathcal{L}_{\mathrm{TL}}
+
\lambda_{CI}\mathcal{L}_{\mathrm{CI}},
\end{equation}
where \(\lambda_{CI}\) controls the strength of cosheaf-inconsistency regularization.

\begin{table}[!t]
\centering
\small
\renewcommand{\arraystretch}{0.8}
\setlength{\tabcolsep}{2.8pt}
\begin{tabular}{lrrrrrrr}
    \toprule
        \multirow{2}{*}{Dataset} & \multicolumn{5}{c}{Dataset Type} & \multicolumn{2}{c}{Tool Graph} \\
        \cmidrule(lr){2-6}\cmidrule(lr){7-8}
        & Single & Chain & DAG & Cycle & Total & $|V|$ & $|E|$ \\
        \midrule
        HuggingFace & 3,067 & 3,642 & 837 & 1,817 & 9,363 & 23 & 225 \\
        Multimedia & 2,037 & 2,982 & 565 & 1,419 & 7,003 & 40 & 449 \\
        Daily Life & 1,258 & 2,787 & 275 & 1,098 & 5,418 & 40 & 1,560 \\
        ToolBench & 101 & 85 & - & 412 & 598 & 245 & 366 \\
    \bottomrule
\end{tabular}
\caption{{\color{black}Dataset and tool-graph statistics.}}
\label{tab:statistics}
\vspace{-1.7em}
\end{table}

\section{Dataset and Data Extension}
\subsection{Dataset}
The proposed framework is evaluated on four tool-planning datasets: HuggingFace, Multimedia, and Daily Life from TaskBench~(\cite{shen2024taskbench}), plus ToolBench~(\cite{chen2026gtool}); their statistics are summarized in Table~\ref{tab:statistics}. HuggingFace focuses on AI model tasks in NLP and computer vision, Multimedia covers tasks such as downloading and format conversion, Daily Life contains everyday service operations such as shopping and route navigation, and ToolBench features a substantially larger and more diverse API inventory. The original queries are mainly categorized as Single, Chain, and DAG. To evaluate topology-aware planning, we further introduce a Cycle category covering loop, merge, and loop-merge workflows. {\color{black}We use \textit{Cycle} as an umbrella category for workflow motifs that induce a closed region in the lifted cellular representation: directed feedback loops, convergent merges formed by multiple dependency paths, and loop-merge compositions. A merge is represented as a cycle-level cell because its multiple dependency paths bound a closed region; this structural notion does not imply repeated execution. For motifs containing a directed loop, the loop denotes bounded revision, reuse, or finite transformation rather than unbounded execution; runtime repetition and termination are governed by the downstream executor or stopping policy.}


\subsection{Agentic Topology-Guided Data Extension}
To evaluate topology-aware task planning, we extend the datasets with queries whose task graphs contain higher-dimensional workflow structures, as illustrated in Figure~\ref{fig:motivation}. Rather than letting a language model freely invent samples, we first fix a target topology and then generate a natural-language request under that structural constraint. {\color{black}This fixes the dependency structure before generation while allowing variation in request wording.}

{\color{black}The extension uses a code-executing construction workflow inspired by code-action agents~(\cite{wang2024executable}): a human designer specifies the structural constraints, executable scripts instantiate candidate workflows, and an LLM generates each request conditioned on the selected tools and fixed topology. Each accepted workflow is represented by task nodes, task links, task steps, and its generated user request.} The resulting samples are validated, merged into the original dataset format, and split with grouped train-test partitioning to reduce structural leakage. All baselines and TopoPlanner are evaluated on this same extended dataset and split; {\color{black}implementation and validation details are provided in Appendix~\ref{app:prompts}, especially Sections~\ref{app:data_workflow} and~\ref{app:leakage}. Appendix~\ref{app:benchmark_validity} further reports a motif-leakage audit, held-out-topology evaluation, and results on original non-augmented workflows.}

\begin{table*}[!t]
    \centering
    \scriptsize
    \setlength{\tabcolsep}{1.7pt}
    \renewcommand{\arraystretch}{0.76}

    \resizebox{0.94\textwidth}{!}{%
     \begin{tabular}{@{}cl|ccc|ccc|ccc|ccc@{}}
     \toprule
     
     \rowcolor{COLOR_MEAN} &   & \multicolumn{3}{c|}{\textbf{HuggingFace}} & \multicolumn{3}{c|}{\textbf{Multimedia}} & \multicolumn{3}{c|}{\textbf{Daily Life}} & \multicolumn{3}{c}{\textbf{ToolBench}} \\
     
     \rowcolor{COLOR_MEAN}  \multirow{-2}{*}{\textbf{LLM}}  & \multirow{-2}{*}{\textbf{Method}} &  \textit{N}  $\uparrow$ & \textit{L}  $\uparrow$ & \textit{Acc} $\uparrow$ & \textit{N}  $\uparrow$ & \textit{L}  $\uparrow$ & \textit{Acc} $\uparrow$ &  \textit{N}  $\uparrow$ & \textit{L}  $\uparrow$ & \textit{Acc} $\uparrow$ & \textit{N}  $\uparrow$ & \textit{L}  $\uparrow$ & \textit{Acc} $\uparrow$  \\

     \midrule
      \multirow{8}{*}{\textbf{Vicuna-7B}} 
        & Direct & 43.66 & 11.91 & 6.70 & 34.78 & 0.45 & 2.05 & 42.82 & 17.05 & 10.11 & 36.83 & 17.89 & 13.24 \\
        & TalkGraph & 57.21 & 26.18 & 8.37 & 43.25 & 13.85 & 3.37 & 58.97 & 29.74 & 20.23 & 20.13 & 12.79 & 5.88 \\
        & Plan-over-Graph  & 57.36 & 11.71 & 9.74 & 43.18 & 0.59 & 3.66 & 60.41 & 16.27 & 22.99 & 36.83 & 21.46 & 13.24 \\
        & GNN4TaskPlan & 57.17 & 27.01 & 8.68 & 46.90 & 17.29 & 4.83 & 58.97 & 29.74 & 20.23 & 24.34 & 18.24 & 11.76 \\
        & {\color{black}USPilot} & \underline{60.67} & \underline{32.31} & \underline{12.79} & 46.45 & 17.20 & 4.83 & 52.10 & 23.91 & 14.71 & 21.09 & 11.18 & 5.88 \\
        & GTool & 52.33 & 19.73 & 8.09 & \underline{47.17} & \underline{19.64} & \underline{7.07} & \underline{75.10} & \underline{52.40} & \underline{42.96} & \underline{45.53} & \underline{28.27} & \underline{30.83} \\
        
        & \cellcolor{blue!10}TopoPlanner  
        & \cellcolor{blue!10}\textbf{82.23} & \cellcolor{blue!10}\textbf{55.72} & \cellcolor{blue!10}\textbf{42.35} 
        & \cellcolor{blue!10}\textbf{86.80} & \cellcolor{blue!10}\textbf{66.24} & \cellcolor{blue!10}\textbf{54.76} 
        & \cellcolor{blue!10}\textbf{97.27} & \cellcolor{blue!10}\textbf{78.53} & \cellcolor{blue!10}\textbf{86.25}
        & \cellcolor{blue!10}\textbf{55.11} & \cellcolor{blue!10}\textbf{35.96} & \cellcolor{blue!10}\textbf{41.67} \\
        
        
     \midrule

     \multirow{8}{*}{\textbf{CodeLlama-7B}} 
        & Direct & 56.19 & 28.88 & 9.12 & 49.69 & 5.27 & 11.23 & 71.58 & 46.09 & 36.19 & 45.63 & 22.60 & 11.67 \\
        & TalkGraph & 61.85 & 32.74 & 9.12 & 52.58 & 29.19 & 13.80 & 69.38 & 44.32 & 32.99 & 44.75 & 29.46 & 21.67 \\
        & Plan-over-Graph  & 61.23 & 24.37 & 9.26 & 56.55 & 4.97 & 16.36 & 73.48 & 45.50 & 38.86 & 45.63 & 24.06 & 11.67 \\
        & GNN4TaskPlan & 61.55 & 32.92 & 9.41 & 57.01 & 34.94 & 18.80 & 71.04 & 45.64 & 35.12 & 41.98 & 26.99 & 20.83 \\
        & {\color{black}USPilot} & 65.06 & 38.49 & 13.82 & 58.59 & 36.75 & 21.61 & 66.59 & 39.76 & 26.17 & 31.02 & 13.56 & 5.00 \\
        & GTool & \underline{68.35} & \underline{39.28} & \underline{23.53} & \underline{69.23} & \underline{46.25} & \underline{30.37} & \underline{79.56} & \underline{58.29} & \underline{52.58} & \underline{55.57} & \underline{38.76} & \underline{36.67} \\
        & \cellcolor{blue!10}TopoPlanner 
        & \cellcolor{blue!10}\textbf{83.43} & \cellcolor{blue!10}\textbf{59.62} & \cellcolor{blue!10}\textbf{47.35} 
        & \cellcolor{blue!10}\textbf{89.08} & \cellcolor{blue!10}\textbf{69.59} & \cellcolor{blue!10}\textbf{62.20} 
        & \cellcolor{blue!10}\textbf{97.43} & \cellcolor{blue!10}\textbf{77.95} & \cellcolor{blue!10}\textbf{87.29}
        & \cellcolor{blue!10}\textbf{59.51} & \cellcolor{blue!10}\textbf{43.88} & \cellcolor{blue!10}\textbf{40.00} \\

        
    \midrule

     \multirow{8}{*}{\textbf{Vicuna-13B}} 
        & Direct & 50.94 & 23.16 & 7.00 & 43.10 & 8.98 & 4.76 & 57.93 & 32.68 & 20.28 & 42.31 & 18.42 & 9.65 \\
        & TalkGraph & 59.47 & 29.31 & 7.88 & 53.28 & 27.43 & 12.29 & 65.56 & 38.00 & 23.95 & 41.53 & 28.94 & 20.18 \\
        & Plan-over-Graph  & 60.90 & 22.22 & 10.50 & 55.58 & 9.32 & 11.74 & 69.64 & 31.93 & 30.42 & 42.31 & 19.67 & 9.65 \\
        & GNN4TaskPlan & 59.90 & 30.95 & 9.88 & 58.11 & 31.61 & 15.06 & 65.56 & 38.00 & 23.95 & 42.82 & \underline{29.14} & 21.93 \\
        & {\color{black}USPilot} & 63.53 & 37.30 & 14.00 & 57.42 & 31.82 & 14.62 & 60.93 & 32.77 & 19.23 & 31.72 & 13.10 & 5.26 \\
        & GTool & \underline{68.98} & \underline{41.03} & \underline{26.03} & \underline{64.87} & \underline{39.53} & \underline{21.71} & \underline{81.49} & \underline{62.14} & \underline{59.11} & \underline{45.07} & 28.52 & \underline{22.50} \\
        
        & \cellcolor{blue!10}TopoPlanner
        & \cellcolor{blue!10}\textbf{83.49} & \cellcolor{blue!10}\textbf{58.32} & \cellcolor{blue!10}\textbf{44.26} 
        & \cellcolor{blue!10}\textbf{89.39} & \cellcolor{blue!10}\textbf{69.20} & \cellcolor{blue!10}\textbf{61.83} 
        & \cellcolor{blue!10}\textbf{98.09} & \cellcolor{blue!10}\textbf{78.74} & \cellcolor{blue!10}\textbf{89.00}
        & \cellcolor{blue!10}\textbf{49.77} & \cellcolor{blue!10}\textbf{31.06} & \cellcolor{blue!10}\textbf{27.66}  \\

    \midrule

     \multirow{8}{*}{\textbf{CodeLlama-13B}} 
        & Direct & 59.70 & 31.12 & 11.47 & 60.82 & 0.34 & 18.54 & 75.78 & 50.92 & 44.41 & 48.00 & 24.91 & 15.00 \\
        & TalkGraph & 63.22 & 36.20 & 8.68 & 60.61 & 36.68 & 19.51 & 73.39 & 49.03 & 38.21 & 43.42 & 28.74 & 22.50 \\
        & Plan-over-Graph  & 64.12 & 27.87 & 11.62 & 63.76 & 0.42 & 20.12 & 76.17 & 49.21 & 44.23 & 48.00 & 25.59 & 15.00 \\
        & GNN4TaskPlan & 64.51 & 37.88 & 10.59 & 64.20 & 40.56 & 23.41 & 73.39 & 49.03 & 38.21 & 40.91 & 27.31 & 21.67 \\
        & {\color{black}USPilot} & 67.89 & 42.99 & 15.59 & 64.75 & 41.18 & 23.66 & 68.26 & 42.72 & 30.81 & 32.03 & 13.00 & 8.33 \\
        & GTool & \underline{71.59} & \underline{47.24} & \underline{33.68} & \underline{73.80} & \underline{52.55} & \underline{39.63} & \underline{82.86} & \underline{63.28} & \underline{60.65} & \underline{57.89} & \underline{38.74} & \underline{35.83} \\
        
        & \cellcolor{blue!10}TopoPlanner 
        & \cellcolor{blue!10}\textbf{83.07} & \cellcolor{blue!10}\textbf{58.21} & \cellcolor{blue!10}\textbf{45.29} 
        & \cellcolor{blue!10}\textbf{90.25} & \cellcolor{blue!10}\textbf{70.93} & \cellcolor{blue!10}\textbf{67.07} 
        & \cellcolor{blue!10}\textbf{97.36} & \cellcolor{blue!10}\textbf{77.08} & \cellcolor{blue!10}\textbf{86.77}
        & \cellcolor{blue!10}\textbf{66.48} & \cellcolor{blue!10}\textbf{49.06} & \cellcolor{blue!10}\textbf{49.17} \\
        
      \bottomrule
    \end{tabular}
    }
      \caption{{\color{black}Results on four benchmarks with local LLM backbones; best values are bold and second-best values are underlined.}}
    \label{tab:result}
    \vspace{-2.0em}
\end{table*}

\section{Experiments}
\subsection{Experimental Setup}

\noindent \textbf{Evaluation Metrics.} Evaluation metrics are adopted from prior work~(\cite{shen2024taskbench,shen2024hugginggpt}): Node F1-Score (\textit{n-F1}) and Link F1-Score (\textit{l-F1}) to measure the accuracy of task predictions and inter-task dependencies, respectively. Additionally, Accuracy (\textit{ACC}) evaluates the overall success rate.

\noindent \textbf{LLMs.} Our main experiments use open-source large language models, including CodeLlama-7B/13B-Instruct-hf~(\cite{Rozire2023CodeLlama}) and Vicuna-7B/13B-v1.5~(\cite{Vicuna}). We additionally evaluate TopoPlanner with API-based LLM backbones, including Qwen3-235B-A22B-Instruct~(\cite{DBLP:journals/corr/abs-2505-09388}) and {\color{black}DeepSeek-V3-671B}~(\cite{DBLP:journals/corr/abs-2412-19437}), in Appendix~\ref{app:large_llm_results}.

\noindent \textbf{Baselines.}
\textbf{Direct:} Queries a pre-trained LLM directly to generate task plans without using any graph structure;
\textbf{TalkGraph}~(\cite{DBLP:conf/iclr/FatemiHP24}): Encodes graph-structured data as text for black-box LLMs, enabling reasoning over nodes and edges through prompt engineering;
\textbf{GNN4TaskPlan}~(\cite{DBLP:conf/nips/WuSSSWZFCCXL24}): Leverages LLMs to interpret user requests and employs GraphSAGE with ranking-loss training to retrieve relevant sub-tasks;
\textbf{Plan-over-Graph}~(\cite{zhang2025plan}): {\color{black}Leverages an LLM} to extract task graphs and applies topological sorting over LLM-predicted task dependencies to generate execution plans;
\textbf{USPilot}~(\cite{DBLP:journals/ral/ChenFCLL25}): Combines a large language model with an LLM-enhanced GNN, using dynamic routing and adapters for task planning;
\textbf{GTool}~(\cite{chen2026gtool}): A strong graph-enhanced tool-planning baseline that encodes the tool-dependency graph with a conventional node-edge GNN, pools the graph into a single soft graph token for the planner LLM, and adds an auxiliary edge-existence objective.

\noindent \textbf{Implementation Details.}
{\color{black}All local-model training and evaluation experiments are conducted on four NVIDIA H800-80G GPUs.} We use \texttt{e5-large-v2}~(\cite{Wang2022e5}) to initialize tool, request, and cellular features. 
TopoPlanner is configured with a hidden dimension of 1024. In the main local-LLM setting, the planner backbone is frozen and only the cellular modules and graph-token interface are trained with AdamW~(\cite{Kingma2014AdamAM}) (\(lr=1\times10^{-3}\), weight decay 0.05, batch size 1) for at most 12 epochs with early stopping (patience 2). For the 7B LoRA continuation experiments, we use a smaller learning rate of \(1\times10^{-5}\). The maximum input and generation lengths are 1024 and 128, respectively. We {\color{black}sweep} the coefficient of \(\mathcal{L}_{CI}\) by varying \(\lambda_{CI}\in\{0, 0.005, 0.01, 0.05, 0.1\}\) across the three TaskBench domains. 

\begin{table}[!t]
    \centering
    \scriptsize
    \tiny
    \setlength{\tabcolsep}{1.1pt}
    \renewcommand{\arraystretch}{0.60}
    \resizebox{0.76\columnwidth}{!}{%
    \begin{tabular}{@{}cc|ccc|ccc|ccc@{}}
    \toprule
    \rowcolor{COLOR_MEAN} &   & \multicolumn{3}{c|}{\textbf{HuggingFace}} & \multicolumn{3}{c|}{\textbf{Multimedia}} & \multicolumn{3}{c}{\textbf{Daily Life}} \\
    \rowcolor{COLOR_MEAN} \multirow{-2}{*}{\textbf{LLM}} & \multirow{-2}{*}{\textbf{Method}}
    & \textit{N} $\uparrow$ & \textit{L} $\uparrow$ & \textit{Acc} $\uparrow$
    & \textit{N} $\uparrow$ & \textit{L} $\uparrow$ & \textit{Acc} $\uparrow$
    & \textit{N} $\uparrow$ & \textit{L} $\uparrow$ & \textit{Acc} $\uparrow$ \\
    \midrule
    \multirow{4}{*}{\textbf{Vicuna-7B}}
        & -w/o 2-cell & 75.58 & 46.61 & 30.88 & 73.29 & 50.13 & 34.15 & 81.25 & 62.42 & 59.28 \\
        & -w/o Retrieval & 76.07 & 47.88 & 32.21 & 77.97 & 55.11 & 35.37 & 82.99 & 65.65 & 64.26 \\
        & -w/o lci & 80.69 & 53.31 & 38.68 & 81.47 & 58.29 & 40.98 & 93.13 & 71.44 & 69.76 \\
        & \cellcolor{blue!10}\textbf{TopoPlanner}
        & \cellcolor{blue!10}\textbf{82.23} & \cellcolor{blue!10}\textbf{55.72} & \cellcolor{blue!10}\textbf{42.35}
        & \cellcolor{blue!10}\textbf{86.80} & \cellcolor{blue!10}\textbf{66.24} & \cellcolor{blue!10}\textbf{54.76}
        & \cellcolor{blue!10}\textbf{97.27} & \cellcolor{blue!10}\textbf{78.53} & \cellcolor{blue!10}\textbf{86.25} \\
    \midrule
    \multirow{4}{*}{\textbf{CodeLlama-7B}}
        & -w/o 2-cell & 78.95 & 51.87 & 35.29 & 75.90 & 56.35 & 43.54 & 82.86 & 64.35 & 63.75 \\
        & -w/o Retrieval & 79.75 & 56.96 & 35.59 & 83.89 & 62.92 & 49.27 & 86.85 & 67.05 & 53.66 \\
        & -w/o lci & 80.8 & 53.83 & 38.82 & 87.12 & 67.61 & 53.54 & 95.41 & 76.2 & 75.95 \\
        & \cellcolor{blue!10}\textbf{TopoPlanner}
        & \cellcolor{blue!10}\textbf{83.43} & \cellcolor{blue!10}\textbf{59.62} & \cellcolor{blue!10}\textbf{47.35}
        & \cellcolor{blue!10}\textbf{89.08} & \cellcolor{blue!10}\textbf{69.59} & \cellcolor{blue!10}\textbf{62.20}
        & \cellcolor{blue!10}\textbf{97.43} & \cellcolor{blue!10}\textbf{77.95} & \cellcolor{blue!10}\textbf{87.29} \\
    \bottomrule
    \end{tabular}
    }
    \caption{{\color{black}TopoPlanner ablations on three datasets with 7B local LLMs.}}
    \label{tab:ablation}
    \vspace{-1em}
\end{table}

\begingroup
\color{black}
\captionsetup[figure]{labelfont={color=black},textfont={color=black}}
\begin{figure}[!t]
    \centering
    \vspace{-0.5em}
    {\scriptsize\textbf{Vicuna-7B}}\par\vspace{0.1em}
    \includegraphics[width=0.78\linewidth]{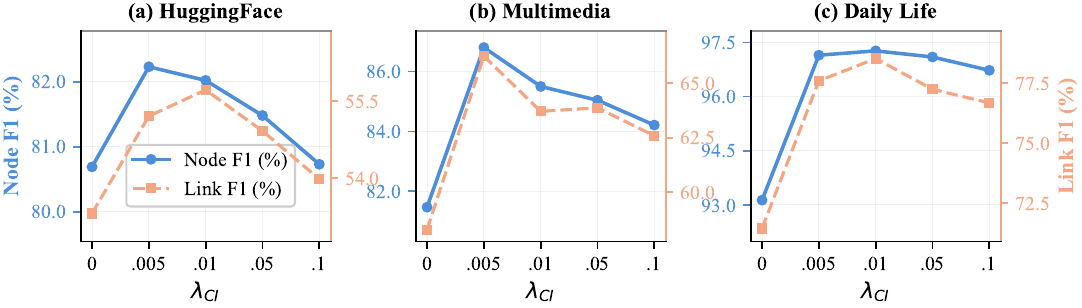}\par\vspace{0.25em}
    {\scriptsize\textbf{CodeLlama-7B}}\par\vspace{0.1em}
    \includegraphics[width=0.78\linewidth]{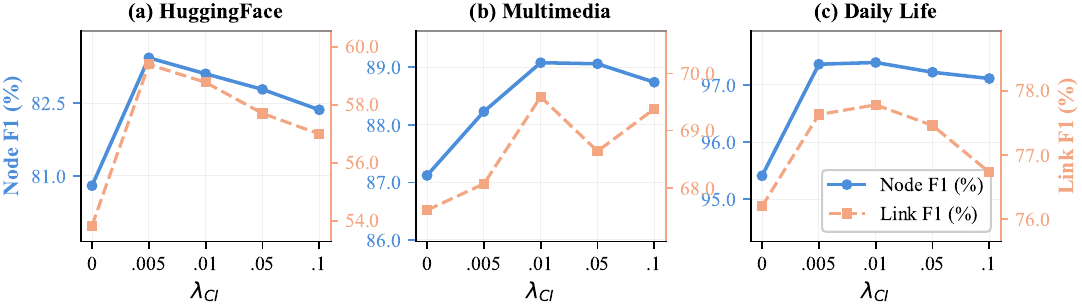}
    \vspace{-0.45em}
    \caption{{\color{black}\(\lambda_{CI}\) sensitivity across datasets and 7B backbones.}}
    \label{fig:sensitivity}
    \vspace{-1.7em}
\end{figure}
\endgroup
\subsection{Experimental Results}

\paragraph{Overall Performance.}
{\color{black}Table~\ref{tab:result} compares methods under open-source local LLMs:}

{\color{black}\textbf{TopoPlanner consistently improves both tool selection and workflow construction.} It outperforms the compared prompting and graph-enhanced baselines across the evaluated datasets and local LLM backbones. More importantly, the improvement is not confined to n-F1: the gains over the strongest baseline are often more pronounced in l-F1 and ACC, which require recovering dependencies and the complete plan rather than only identifying relevant tools. For example, on Multimedia with Vicuna-7B, TopoPlanner improves over GTool from 47.17/19.64/7.07 to 86.80/66.24/54.76 in n-F1/l-F1/ACC; on the larger ToolBench graph with CodeLlama-13B, the corresponding scores increase from 57.89/38.74/35.83 to 66.48/49.06/49.17.}

{\color{black}\textbf{These results align with the structural gap motivating TopoPlanner.} The strongest comparator in most settings, GTool, already conditions the same planner backbones on a pooled node--edge graph representation. TopoPlanner instead makes loop, merge, and loop-merge regions explicit as 2-cells, retrieves their request-relevant boundary-closed subcomplex, and propagates information across dimensions before conditioning the planner. Its additional gains over GTool are therefore consistent with the value of preserving the regional structures highlighted in Figure~\ref{fig:motivation}, rather than flattening them into pairwise relations. The advantage persists across the 7B and 13B Vicuna and CodeLlama families, suggesting that this structural interface complements backbone capacity.}

{\color{black}\textbf{To isolate the sources of these gains,} the ablations below evaluate 2-cell lifting, request-specific cellular retrieval, and its consistency objective. Appendix~\ref{app:component_diagnostics} tests boundary closure and an alternative hypergraph representation, while Appendix~\ref{app:structure_diagnostics} directly evaluates merge/reuse recall and dependency order. Together, these diagnostics test the proposed regional representation and boundary-complete retrieval.}
\vspace{-0.5em}
\paragraph{Ablation Study.}
{\color{black}Table~\ref{tab:ablation} compares the full model with \texttt{-w/o 2-cell}, which uses the edge-only 1-skeleton; \texttt{-w/o Retrieval}, which removes request-specific closed-subcomplex retrieval; and \texttt{-w/o lci}, which removes the retrieval-consistency supervision.}
\vspace{-0.5em}
\begin{itemize}
    \item {\color{black}\textbf{Explicit higher-order lifting provides useful regional context.} Replacing the cellular workflow complex with its 1-skeleton causes a clear performance drop across both backbones and all three domains. Under the evaluated architectures, explicit 2-cell regions provide context for loop, merge, and loop-merge patterns beyond an edge-only model.}

    \item {\color{black}\textbf{Cellular retrieval and its consistency signal are complementary.} Relative to removing \(\mathcal{L}_{CI}\), removing retrieval causes a larger drop in 17 of the 18 backbone--domain--metric comparisons; the sole exception is CodeLlama-7B/HuggingFace l-F1. The full model exceeds both variants in all 18 comparisons.}
\end{itemize}

\begin{figure}[!t]
	\centering
	\includegraphics[width=0.95\linewidth]{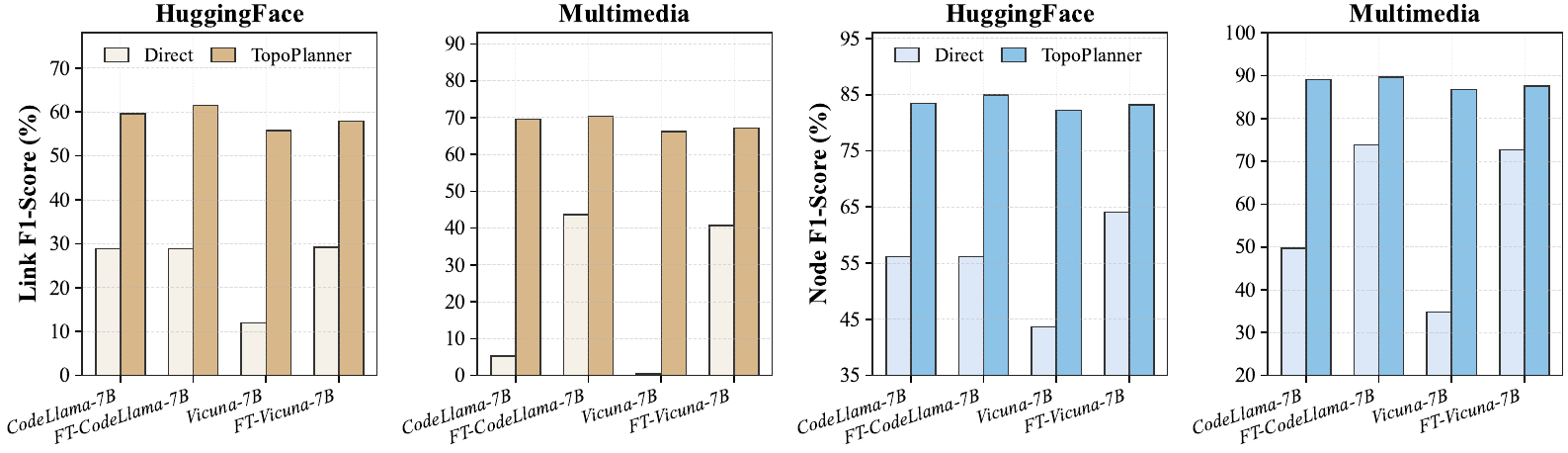}
	\vspace{-0.45em}
	\caption{{\color{black}One-epoch LoRA continuation results.}}
	\label{fig:finetuning}
	\vspace{-1.1em}
\end{figure}

\vspace{-0.5em}
{\color{black}\noindent\textbf{Sensitivity to \(\lambda_{CI}\).} Figure~\ref{fig:sensitivity} probes the consistency term in cellular retrieval, which is designed to keep the request-relevant subcomplex boundary-compatible rather than selecting cells independently. All six panels peak at a positive \(\lambda_{CI}\), usually 0.005 or 0.01, whereas larger weights generally plateau or reduce performance. This pattern supports the intended balance of TopoPlanner: regional consistency complements semantic relevance but should not dominate it.}

{\color{black}\noindent\textbf{Effect of fine-tuning.} Figure~\ref{fig:finetuning} tests whether backbone adaptation alone closes the structural gap highlighted in Figure~\ref{fig:motivation}. One-epoch LoRA improves Direct, especially on Multimedia, but the corresponding TopoPlanner result remains higher before and after adaptation in all four panels. The result is consistent with the two signals being complementary: fine-tuning strengthens the planner, while cellular conditioning preserves explicit information about non-linear dependency regions. This conclusion is limited to the tested one-epoch setting.} {\color{black}Qualitative analyses of loop, merge, and loop-merge workflows are provided in Appendix~\ref{app:qualitative_analysis}.}

\section{{\color{black}Related Work}}
Task planning in LLM-based agents aims to decompose natural-language requests into executable sub-tasks. Most existing methods are prompt-based, including task decomposition~(\cite{wei2022chain}), multi-plan search~(\cite{yao2024tree,besta2024graph}), symbolic planning~(\cite{liu2023llm+}), and reflection-based refinement~(\cite{shinn2024reflexion}). These approaches have broadened the planning ability of LLMs, but they usually assume sequential, tree-like, or DAG-like workflows, which makes it difficult to capture richer dependency patterns such as loops, merges, and feedback structures.

Another line of work studies graph reasoning with LLMs. GPT4Graph and NLGraph evaluate how pretrained LLMs process graph problems in natural language~(\cite{guo2023gpt4graph,wang2024can}), while TalkLikeAGraph and Plan-over-Graph encode graphs as text and rely on the LLM to reason over dependencies or derive execution orders~(\cite{DBLP:conf/iclr/FatemiHP24,zhang2025plan}). These results suggest that LLMs can exploit graph-structured context, but text-only reasoning still lacks explicit structural bias on more complex topological relations.

To address this limitation, several hybrid LLM+GNN methods inject graph representations into LLMs. GraphGPT, GraphLLM, and GraphToken incorporate learned graph embeddings or graph tokens into the language-model input on graph-centric tasks~(\cite{tang2024graphgpt,chai2023graphllm,perozzi2024graphtoken}). In tool planning, GNN4TaskPlan~(\cite{DBLP:conf/nips/WuSSSWZFCCXL24}) uses an LLM to decompose user goals and a GNN matcher to retrieve relevant sub-tasks, while USPilot~(\cite{DBLP:journals/ral/ChenFCLL25}) combines an LLM with an LLM-enhanced GNN, dynamic routing, and adapters. GTool~(\cite{chen2026gtool}) encodes the tool-dependency graph with a node-edge GNN, pools it into a single graph token for the planner LLM, and adds an auxiliary edge-existence objective. Despite their strength, these methods still reason mainly over node-edge structure and do not explicitly elevate higher-order workflow regions into the planning representation. {\color{black}Appendix~\ref{app:related} reviews topological deep learning.}

\section{Conclusion}
{\color{black}In this work, we present \textbf{TopoPlanner}, a topology-consistent task planning framework that makes loop, merge, and loop-merge regions explicit by lifting tool dependency graphs into cellular workflow complexes. The lifted structure supports request-specific retrieval, multi-dimensional reasoning, and topology-guided tool-sequence generation. We also construct topology-guided extensions of four tool-planning benchmarks. Across the evaluated datasets and local LLM backbones, TopoPlanner consistently outperforms the compared baselines; the ablations and structure-specific diagnostics further support explicit regional modeling as an effective inductive bias within the evaluated scope.}

\subsection*{{\color{black}AI use statement}}
{\color{black}Generative AI tools assisted with language editing and implementation of the topology-guided data-construction pipeline; an LLM also generated requests under fixed workflow constraints, as documented in Appendix~\ref{app:data_workflow}. All outputs and scientific claims were reviewed and verified by the authors.}

\bibliographystyle{iclr2027_conference}
\bibliography{reference}

\appendix
\clearpage
\appendix

\section{Proofs}
\subsection{Proofs regarding Proposition \ref{prop.homotopy}}
\label{app.proof.homotopy}
\begin{proof}
\textbf{Contractibility of the spanning tree.}  
{\color{black}For the connected underlying undirected graph $G_u$, a spanning tree $\mathcal{T}$ is connected and acyclic, hence contractible.}
{\color{black}A contractible subspace can be continuously shrunk to a point.}

\textbf{Collapsing a contractible subspace.}  
{\color{black}Consider the quotient map $\gamma:G_u\to G_u/\mathcal{T}$ that identifies all points of $\mathcal{T}$ with a single vertex $v_0$. Because $\mathcal{T}$ is a contractible subcomplex of the graph, the quotient map is a homotopy equivalence. Hence there exists a homotopy inverse $s:G_u/\mathcal{T}\to G_u$ such that $s\circ\gamma\simeq\mathrm{id}_{G_u}$ and $\gamma\circ s\simeq\mathrm{id}_{G_u/\mathcal{T}}$. Thus, $G_u\simeq G_u/\mathcal{T}$.}

\textbf{Induced isomorphism on first homology.}  
Homotopy-equivalent spaces have isomorphic homology groups.  
Hence the induced map 
$\gamma_*: H_1(G_u;\mathbb{Z}) \to H_1(G_u/\mathcal{T};\mathbb{Z})$ 
is an isomorphism.

\textbf{Intuition in graph terms.}  
{\color{black}The spanning tree $\mathcal{T}$ contains no cycles, so collapsing it does not remove or merge any cycles in $G_u$. Each fundamental cycle in $G_u$ (formed by a non-tree edge and the unique tree path connecting its endpoints) becomes a loop based at the collapsed tree vertex. Thus, collapsing $\mathcal{T}$ preserves the independent cycles measured by $H_1(G_u)$.}
\end{proof}

\subsection{Proof of the Fundamental-Cycle Basis Claim}
\label{app.proof.fundamental_cycles}
\begin{proof}
{\color{black}Let $G_u=(V,E)$ be a finite connected undirected graph underlying a directed workflow graph, and let $\mathcal{T}\subset G_u$ be a spanning tree. Edge directions remain attributes for planning but do not alter the cycle-space argument.}

\textbf{Existence and uniqueness of fundamental cycles.}  
For each non-tree edge $e=(u,v) \in E \setminus \mathcal{T}$, there exists a unique simple path $P_{\mathcal{T}}(u,v)$ in $\mathcal{T}$ connecting $u$ and $v$, by acyclicity of $\mathcal{T}$.  
Concatenating $e$ with $P_{\mathcal{T}}(u,v)$ defines a unique simple cycle 
\[
C_e = e \cup P_{\mathcal{T}}(u,v) \subset G.
\]  
{\color{black}Under the quotient map $\gamma:G_u\to G_u/\mathcal{T}$, the tree path $P_{\mathcal{T}}(u,v)$ is mapped to that point, so $C_e$ becomes a nontrivial loop in $G_u/\mathcal{T}$.}

\textbf{Spanning and independence in homology.}  
{\color{black}The cyclomatic number of $G_u$ is $\beta_1(G_u)=|E|-|V|+1$, which equals the number of non-tree edges.}
Thus there are exactly $\beta_1(G)$ fundamental cycles.  

{\color{black}Any cycle in $G_u$ is a linear combination of these fundamental cycles: whenever traversal encounters a non-tree edge $e$, the corresponding $C_e$ eliminates the associated tree segment.}

{\color{black}These cycles are independent in $H_1(G_u)$: under $\gamma$, each fundamental cycle maps to a distinct loop in $G_u/\mathcal{T}$, and loops around different edges of a wedge of circles are homologically independent.}

Hence the set 
\[
\{ [C_e] \mid e \in E \setminus \mathcal{T} \}
\]
{\color{black}forms a basis of $H_1(G_u)$ under the stated connected-graph assumption.}

\textbf{Topological summary.}  
{\color{black}Different choices of spanning tree yield different fundamental-cycle edge sets, but their homology classes always span $H_1(G_u)$. These loops therefore capture all independent cycle-space relations of the underlying graph, while the retained edge directions preserve the workflow semantics used by the planner.}
\end{proof}

\begingroup
\color{black}
\section{Theoretical Role of the Cellular Structure}
\label{app:cellular_theory}

The homotopy and cycle-basis results above justify the lifting construction. We next clarify what the lifted cells contribute to representation learning and retrieval. These statements characterize the structural inductive bias of TopoPlanner; boundary closure alone does not certify runtime correctness.

\begin{proposition}[Fixed-depth regional communication]
Let $R$ be a workflow region containing two boundary elements whose graph distance exceeds $L$ in the pairwise 1-skeleton. In an $L$-layer local message-passing network, the representation of either element cannot depend on the other. If $R$ is represented by a shared 2-cell, every boundary element is incident to that cell, and information can be exchanged across the region through a boundary-to-coface update followed by a coface-to-boundary update, regardless of boundary length.
\end{proposition}
\begin{proof}
Local pairwise message passing expands a receptive field by at most one graph hop per layer. Hence two elements at distance greater than $L$ cannot influence one another after $L$ layers. After lifting $R$ to a 2-cell, each boundary element is connected to the shared coface through an incidence relation. A boundary-to-coface update followed by a coface-to-boundary update therefore communicates across the region in constant depth. This is a representational shortcut rather than a correctness guarantee; a suitably constructed hypergraph can provide a related shortcut, while the cellular representation additionally retains graded boundary relations.
\end{proof}

\begin{proposition}[Incidence-consistency energy]
For a boundary-closed candidate $\mathcal{Y}$, let $H(\mathcal{Y})$ be its directed Hasse incidence graph and let the normalized coface-to-boundary residual be
\begin{equation}
 (\delta_q \bm h)_{\sigma,\tau}
 =
 \widehat{\mathcal{C}^{q}_{\sigma\rightarrow\tau}(\bm h_\sigma)}
 - \widehat{\bm h_\tau},
 \qquad \tau\prec\sigma.
\end{equation}
Then the cosine inconsistency in Eq.~(8) is a weighted quadratic incidence energy because $1-\cos(\bm a,\bm b)=\frac{1}{2}\|\widehat{\bm a}-\widehat{\bm b}\|_2^2$. In the quadratic implementation, the same construction has the form
\begin{equation}
 E_q(\mathcal{Y},\bm h)
 = \|W_q^{1/2}\delta_q\bm h\|_2^2
 = \bm h^\top L_q\bm h,
 \qquad
 L_q=\delta_q^\top W_q\delta_q.
\end{equation}
Consequently, low energy indicates proximity to assignments that agree across the selected incidences.
\end{proposition}

\begin{proposition}[Closure prevents consistency by omission]
Let $S$ be a selected set of cells, $\operatorname{cl}(S)$ its boundary closure, and $I(S)$ the set of coface--boundary incidences whose two cells are both contained in $S$. With non-negative incidence weights and per-incidence residuals $e_q(\sigma,\tau)\geq 0$,
\begin{equation}
E_q(\operatorname{cl}(S),\bm h)
=E_q(S,\bm h)+
\sum_{(\sigma,\tau)\in I(\operatorname{cl}(S))\setminus I(S)}
w_{\sigma\tau}e_q(\sigma,\tau)
\geq E_q(S,\bm h).
\end{equation}
Thus closure forces the objective to evaluate newly included incidences between retained cofaces and their required boundary cells, rather than treating the added cells as an independent subcomplex. Moreover, let $\mathcal{Y}^{\star}$ be a reachable closed reference workflow and $\mathcal{Y}$ a competing closed workflow. If the competitor's relevance advantage is at most $\rho$ while its excess incidence inconsistency is at least $\mu$, then for $J_q(\mathcal{Y})=\bar r_q(\mathcal{Y})-\bar\epsilon_q(\mathcal{Y})$,
\begin{equation}
J_q(\mathcal{Y}^{\star})-J_q(\mathcal{Y})\geq \mu-\rho.
\end{equation}
Hence $\mu>\rho$ is a sufficient conditional margin for preferring the incidence-compatible workflow. This criterion explains consistency and closure only under its stated assumptions.
\end{proposition}

\endgroup

\section{Algorithm Details}
\label{algorithm}
Based on the theoretical proofs provided in Section \ref{app.proof.fundamental_cycles}, we detail the computational procedure for lifting the task graph into a cellular complex. Algorithm \ref{alg:fundamental_cycles} outlines the process of extracting the fundamental cycle basis using a spanning tree approach. This method guarantees a linear number of cycles relative to the cyclomatic number $\beta_1$, ensuring the scalability discussed in Section \ref{app:complexity}.
\begin{algorithm}[H]
\caption{Fundamental Cycle Basis Construction via Spanning Tree}
\label{alg:fundamental_cycles}
\begin{algorithmic}[1]
\REQUIRE {\color{black}Connected underlying undirected graph $G_u=(V,E)$; optional maximum length $K_{max}$}
\ENSURE {\color{black}Fundamental cycle set $\mathcal{C}$ (a complete basis when $K_{max}$ is None)}

\STATE \textbf{Initialize:} Disjoint Set Union (DSU) structure for cycle detection.
\STATE $\mathcal{T} \leftarrow \emptyset$ \COMMENT{Initialize Spanning Tree edges}
\STATE $E_{non} \leftarrow \emptyset$ \COMMENT{Initialize Non-Tree edges}
\STATE $\mathcal{C} \leftarrow \emptyset$ \COMMENT{Initialize Cycle Basis}

\STATE \COMMENT{\textbf{Phase 1: Build Spanning Tree (Kruskal-like logic)}}
\FOR{each edge $(u, v) \in E$}
    \IF{$\textsc{FindSet}(u) \neq \textsc{FindSet}(v)$}
        \STATE $\mathcal{T} \leftarrow \mathcal{T} \cup \{(u, v)\}$
        \STATE $\textsc{UnionSet}(u, v)$
    \ELSE
        \STATE $E_{non} \leftarrow E_{non} \cup \{(u, v)\}$
    \ENDIF
\ENDFOR

\STATE \COMMENT{\textbf{Phase 2: Extract Fundamental Cycles}}
\STATE Construct adjacency list for $\mathcal{T}$ to enable path traversal.
\FOR{each non-tree edge $e=(u, v) \in E_{non}$}
    \STATE \COMMENT{Find the unique path in $\mathcal{T}$ connecting $u$ and $v$}
    \STATE $P_{\mathcal{T}}(u, v) \leftarrow \textsc{BFS}(start=u, target=v, graph=\mathcal{T})$
    
    \STATE \COMMENT{Form the cycle by combining the non-tree edge and the tree path}
    \STATE $C_e \leftarrow P_{\mathcal{T}}(u, v) \cup \{e\}$
    
    \IF{$K_{max}$ is None \OR $|C_e| \leq K_{max}$}
        \STATE $\mathcal{C} \leftarrow \mathcal{C} \cup \{C_e\}$
    \ENDIF
\ENDFOR

\RETURN $\mathcal{C}$
\end{algorithmic}
\end{algorithm}
{\color{black}The homology and basis claims use the untruncated setting $K_{max}=\mathrm{None}$. With a finite $K_{max}$, the algorithm returns only the retained length-filtered subset, which is not claimed to span the full cycle space.}
\section{Data Construction and Prompt Details}
\label{app:prompts}

To ensure reproducibility and facilitate future research, we describe the topology-guided data construction workflow and provide the prompt template used by TopoPlanner during inference. The data extension is not produced by a one-shot generation prompt. Instead, it is generated by executable scripts developed through an agentic code-writing workflow, following the idea of code-action agents~(\cite{wang2024executable}) and agentic data construction~(\cite{mitra2024agentinstruct}).

\subsection{Agentic Topology-Guided Construction Workflow}
\label{app:data_workflow}
{\color{black}Given human-specified workflow-topology requirements, Codex is used as a code-executing agent to implement and iterate the data construction pipeline. Algorithm~\ref{alg:data_construction} summarizes the process: it mines valid motifs from the original tool graph, instantiates samples, generates requests under fixed topology constraints, validates them, and creates leakage-controlled splits.}

\paragraph{Structure constraints.}
The construction code treats the original tool graph as the source of structural validity. For each dataset, it loads the tool inventory and directed dependency links, builds adjacency lists and an edge lookup table, and samples candidate motifs only from this graph. We use three concrete paradigms: \textit{{\color{black}directed loop}}, such as \(A\rightarrow B\rightarrow C\rightarrow A\); \textit{merge}, where two dependency paths converge to a shared downstream tool, such as \(A\rightarrow C\) and \(A\rightarrow B\rightarrow C\); and \textit{{\color{black}loop-merge}}, where a merge structure is further connected by a closing path back to an earlier tool. A candidate is accepted only if all of its workflow dependencies are present in the original edge set, all task nodes belong to the tool inventory, and the resulting task graph satisfies the intended topological pattern. {\color{black}Thus, topology is fixed before request generation.}

\paragraph{Codex workflow instructions.}
Codex is not prompted to directly invent final samples. Instead, it receives module-level implementation instructions and iteratively edits executable construction programs. The main instructions used in the workflow are summarized below:
\begin{itemize}
    \item {\color{black}\textbf{Motif mining:} ``Sample valid directed-loop, merge, and loop-merge motifs from each tool graph.''}
    \item \textbf{Sample construction:} ``For each valid motif, construct task nodes, task links, and task steps in the original dataset schema. Record motif statistics for later inspection.''
    \item {\color{black}\textbf{Request generation:} ``Generate a natural request from the fixed, hidden tool topology.''}
    \item {\color{black}\textbf{Validation and merging:} ``Validate tools, edges, node-link consistency, identifiers, and motifs; repair or reject errors before merging leakage-controlled splits.''}
\end{itemize}

\begin{algorithm}[h]
\caption{Agentic Topology-Guided Data Construction}
\label{alg:data_construction}
\begin{algorithmic}[1]
\REQUIRE Dataset collection $\mathcal{D}$, target counts $\{N_d\}$, topology types $\mathcal{C}$
\ENSURE Extended datasets with train/test splits
\FOR{each dataset $d \in \mathcal{D}$}
    \STATE Load tool graph $G_d=(\mathcal{T}_d,\mathcal{E}_d)$ and tool descriptions
    \STATE Build adjacency lists and valid edge lookup from $\mathcal{E}_d$
    \STATE $\mathcal{S}_d \leftarrow \emptyset$
    \WHILE{$|\mathcal{S}_d| < N_d$}
        \STATE Sample a topology type $c \in \mathcal{C}$
        \STATE Mine a candidate motif $m$ of type $c$ from $G_d$
        \IF{$m$ is not structurally valid under $\mathcal{E}_d$}
            \STATE continue
        \ENDIF
        \STATE Instantiate $m$ as task nodes, task links, and task steps
        \STATE Generate a user request conditioned on tools, descriptions, and motif $m$
        \STATE Form candidate sample $s$
        \IF{\textsc{ValidateSample}$(s,\mathcal{T}_d,\mathcal{E}_d,c)$ accepts}
            \STATE $\mathcal{S}_d \leftarrow \mathcal{S}_d \cup \{s\}$
        \ENDIF
    \ENDWHILE
    \STATE Group samples by request, task nodes, and task links
    \STATE Merge $\mathcal{S}_d$ into the original dataset and assign leakage-controlled splits
\ENDFOR
\RETURN Extended datasets
\end{algorithmic}
\end{algorithm}

\paragraph{Validation algorithm.}
The final validator is deterministic and is applied after request generation. Algorithm~\ref{alg:data_validation} summarizes the logic.

\begin{algorithm}[h]
\caption{Topology-Guided Sample Validation}
\label{alg:data_validation}
\begin{algorithmic}[1]
\REQUIRE Candidate sample $s$, tool set $\mathcal{T}$, valid edge set $\mathcal{E}$, target topology type $c$
\ENSURE Accept or reject $s$
\STATE Extract task nodes $V_s$ and task links $E_s$ from $s$
\IF{any node in $V_s$ is not in $\mathcal{T}$}
    \RETURN Reject
\ENDIF
\IF{any link in $E_s$ is not in $\mathcal{E}$}
    \RETURN Reject
\ENDIF
\IF{task steps, task nodes, and task links are structurally inconsistent}
    \RETURN Reject
\ENDIF
\IF{$s$ does not satisfy the target topological pattern $c$}
    \RETURN Reject
\ENDIF
\IF{$s$ duplicates an existing identifier or an existing workflow key}
    \RETURN Repair or reject
\ENDIF
\RETURN Accept
\end{algorithmic}
\end{algorithm}

\subsection{Human Evaluation Protocol}
\label{app:human_eval}
We further conduct a human evaluation to assess the quality of the topology-guided data extension. Following the evaluation dimensions used in TaskBench~(\cite{shen2024taskbench}), each sample is rated from 1 to 5 along three axes: \textbf{Naturalness}, which measures whether the user request is fluent and realistic; \textbf{Complexity}, which measures whether the task requires non-trivial multi-step tool use; and \textbf{Alignment}, which measures whether the request semantics match the annotated task graph.

To reduce single-perspective bias, we invite three human evaluators with different backgrounds: one task-planning researcher who focuses on structural correctness, one domain practitioner who focuses on tool semantics, and one user-facing reviewer who focuses on request naturalness. For each dataset, evaluators inspect a randomly sampled subset of the augmented Cycle examples together with their task nodes and task links. The final score is the average over evaluators and sampled examples. {\color{black}These values are descriptive expert ratings rather than inferential population estimates.} This expert assessment is used only to rate benchmark instances; it does not involve crowdsourcing, collection of personal data, or interventions on participants. The results are reported in Table~\ref{tab:evaluation}.

\begin{table}[!t]
    \centering
    \small
    \setlength{\tabcolsep}{3pt}
    \begin{tabular}{ccccc}
    \toprule
    \textbf{Dataset} & \textit{Naturalness}~$\uparrow$ & \textit{Complexity}~$\uparrow$ & \textit{Alignment}~$\uparrow$ & \textit{Overall}~$\uparrow$ \\
    \midrule
    \textbf{HuggingFace} & 4.28 & 4.34 & 4.16 & 4.26 \\
    \textbf{Multimedia} & 4.35 & 4.29 & 4.22 & 4.29 \\
    \textbf{Daily Life} & 4.06 & 4.21 & 3.88 & 4.05 \\
    \bottomrule
    \end{tabular}
    \caption{{\color{black}Human ratings (1--5) for the topology-guided data extension.}}
    \label{tab:evaluation}
\end{table}

\subsection{Measures to Prevent Data Leakage}
\label{app:leakage}
We strictly isolate the training and evaluation phases to prevent data contamination:
\begin{itemize}
    \item \textbf{Offline construction:} The data extension is performed before model training and evaluation. No generation-time construction state is exposed to the evaluated planner.
    \item \textbf{Grouped splitting:} Augmented samples are grouped by normalized request, ordered task nodes, and dependency links before train-test assignment. This reduces the risk that near-duplicate workflow graphs appear across training and testing splits.
    \item \textbf{Format preservation:} The augmented samples are merged into the original dataset schema, and all evaluated methods use the same updated split.
\end{itemize}

\subsection{Model Overview and API-Compatible Interface}
\label{app:model_and_api_interface}

Figure~\ref{fig:model} in the main text summarizes the local open-source TopoPlanner pipeline. Here we clarify the API deployment mode. When the backbone is accessed only through an API, the model exposes text completions but not token embeddings, so the local graph-token interface in the main model cannot be applied directly. We therefore keep the same cellular lifting, retrieval, and multi-dimensional reasoning modules locally, use the API LLM only for textual proposal generation, and let the local topology module perform the final topology-guided selection and graph-consistent decoding.

\begin{figure*}[t]
    \centering
    \includegraphics[width=\textwidth]{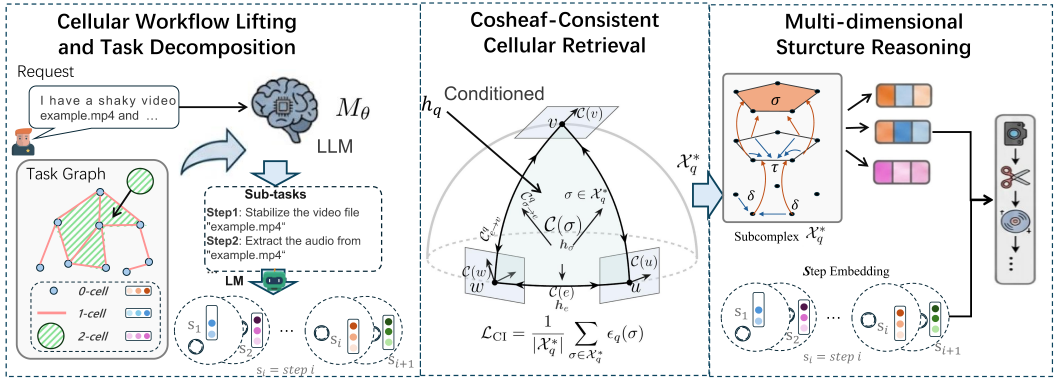}
    \caption{{\color{black}API-compatible TopoPlanner with local topology-guided proposal selection.}}
    \label{fig:model_api}
\end{figure*}

\subsection{TopoPlanner Inference Prompts}
\label{app:inf_prompts}

\paragraph{Local graph-token prompt.}
For local open-source LLMs, TopoPlanner follows the same output interface as the evaluated tool-planning benchmarks: the LLM produces one tool name per line, and dependency edges are reconstructed by the shared evaluation protocol. Given a request \(q\), TopoPlanner retrieves a closed cellular subcomplex \(\mathcal{X}^{*}_q\), encodes it into a graph token \(\bm{h}_{\mathcal{X}_q}\), and injects it at the \texttt{[graph]}...\texttt{[/graph]} position. The textual part of the prompt is:

\begin{center}
\fbox{
  \begin{minipage}{0.9\linewidth}
    \small
    \textbf{Local Planning Prompt:} \\
    \# GOAL \#: Please use the provided tools to solve the problem. \\
    \\
    \# GRAPH CONTEXT \#: \texttt{[graph]} \(\bm{h}_{\mathcal{X}_q}\) \texttt{[/graph]} \\
    \\
    \# TASK LIST \#: \{tool\_list\} \\
    \# USER REQUEST \#: \{user\_request\} \\
    \\
    \# OUTPUT FORMAT \#: Output only one exact tool name per line: \\
    Tool1: tool\_name \\
    Tool2: tool\_name
  \end{minipage}
}
\end{center}

\paragraph{API-accessible LLMs.}
For Qwen3-235B and DeepSeek-V3, the interface is text-only and does not expose token embeddings or support replacing a placeholder with \(\bm{h}_{\mathcal{X}_q}\). We therefore do not claim graph-token injection for these models. Instead, the API LLM predicts textual task steps, candidate tools, and coarse task links from the plain-text tool list and user request, while the local lifted cellular encoder performs topology-guided tool scoring and graph-consistent decoding on top of these proposals, as illustrated in Figure~\ref{fig:model_api}. We also do not include GTool as an API baseline, because its graph-enhanced formulation likewise assumes direct model-side graph/embedding injection, which is unavailable under the same black-box API constraint.

\section{Additional Results on Large LLMs}
\label{app:large_llm_results}
Table~\ref{tab:large_llm_result} reports the main results on API-accessible large LLMs, including Qwen3-235B and DeepSeek-V3-671B, using the API interface described above. These numbers should therefore be interpreted as gains from topology-guided proposal selection and decoding over black-box textual outputs, rather than as hidden-state graph-token injection into the API models themselves.

\begin{table*}[!t]
    \centering
    \tiny
    \renewcommand{\arraystretch}{0.75}

    \resizebox{\textwidth}{!}{%
     \begin{tabular}{cc|ccc|ccc|ccc}
     \toprule
     
     \rowcolor{COLOR_MEAN} &   & \multicolumn{3}{c|}{\textbf{HuggingFace}} & \multicolumn{3}{c|}{\textbf{Multimedia}} & \multicolumn{3}{c}{\textbf{Daily Life}} \\
     
     \rowcolor{COLOR_MEAN}  \multirow{-2}{*}{\textbf{LLM}}  & \multirow{-2}{*}{\textbf{Method}} &  \textit{n-F1}  $\uparrow$ & \textit{l-F1}  $\uparrow$ & \textit{Acc} $\uparrow$ & \textit{n-F1}  $\uparrow$ & \textit{l-F1}  $\uparrow$ & \textit{Acc} $\uparrow$ &  \textit{n-F1}  $\uparrow$ & \textit{l-F1}  $\uparrow$ & \textit{Acc} $\uparrow$  \\ 

     \midrule
      \multirow{6}{*}{\textbf{Qwen3-235B}} 
        & Direct & 73.63 & 47.77 & 23.82 & 77.10 & 57.49 & 48.29 & 83.96 & 62.88 & 67.36 \\
        & TalkGraph & 60.03 & 28.59 & 10.00 & 62.06 & 39.30 & 22.32 & 78.10 & 57.44 & 50.48 \\
        & Plan-over-Graph  & 62.54 & 31.59 & 9.85 & 66.17 & 42.03 & 22.93 & 80.73 & 59.22 & 58.60 \\
        & GNN4TaskPlan & 61.99 & 31.25 & 11.91 & 68.45 & 47.04 & 30.12 & 78.10 & 57.44 & 50.48 \\
        & {\color{black}USPilot} & 70.48 & 44.64 & 21.03 & 68.28 & 46.33 & 31.22 & 73.74 & 50.78 & 41.08 \\
        & \cellcolor{blue!10}TopoPlanner 
        & \cellcolor{blue!10}\textbf{81.40} & \cellcolor{blue!10}\textbf{60.56} & \cellcolor{blue!10}\textbf{44.74} 
        & \cellcolor{blue!10}\textbf{87.65} & \cellcolor{blue!10}\textbf{71.00} & \cellcolor{blue!10}\textbf{61.33} 
        & \cellcolor{blue!10}\textbf{97.95} & \cellcolor{blue!10}\textbf{83.16} & \cellcolor{blue!10}\textbf{89.65} \\    

        \midrule

      \multirow{6}{*}{\textbf{DeepSeek-V3-671B}} 
        & Direct & 76.42 & 53.82 & 27.65 & 77.61 & 56.09 & 47.80 & 82.76 & 47.15 & 67.01 \\
        & TalkGraph & 63.50 & 35.80 & 10.88 & 64.43 & 40.39 & 24.76 & 77.02 & 55.82 & 49.66 \\
        & Plan-over-Graph  & 67.33 & 39.84 & 13.82 & 68.89 & 44.72 & 25.98 & 79.36 & 43.44 & 58.25 \\
        & GNN4TaskPlan & 65.32 & 37.34 & 12.65 & 70.06 & 47.47 & 31.83 & 77.02 & 55.82 & 49.66 \\
        & {\color{black}USPilot} & 72.01 & 47.89 & 23.53 & 69.07 & 45.51 & 30.37 & 73.31 & 49.93 & 41.41 \\
        & \cellcolor{blue!10}TopoPlanner 
        & \cellcolor{blue!10}\textbf{82.21} & \cellcolor{blue!10}\textbf{59.86} & \cellcolor{blue!10}\textbf{45.79} 
        & \cellcolor{blue!10}\textbf{88.08} & \cellcolor{blue!10}\textbf{69.19} & \cellcolor{blue!10}\textbf{61.17} 
        & \cellcolor{blue!10}\textbf{98.33} & \cellcolor{blue!10}\textbf{82.95} & \cellcolor{blue!10}\textbf{91.43} \\   
      \bottomrule
    \end{tabular}
    }
    \vspace{-0.5em}
      \caption{{\color{black}API-backbone results with the API-compatible topology interface.}}
    \label{tab:large_llm_result}
    \vspace{-1.5em}
\end{table*}

\section{Limitations, Broader Impacts, and Asset Use}
\label{app:limitations_broader}

\paragraph{Limitations.}
{\color{black}Our empirical evaluation covers four benchmark-style tool-planning datasets; unconstrained real-world deployments remain unevaluated. The lifting mechanism targets loop, merge, and loop-merge workflow regions, and \(\lambda_{CI}\) is selected per dataset. In the black-box API setting, TopoPlanner uses topology-guided proposal selection and decoding because graph-token injection is unavailable.}

\paragraph{Broader impacts.}
{\color{black}Topology-consistent planning can improve dependency preservation in workflows with feedback, verification, or branch merging. When connected to external APIs or action tools, the same capability can increase execution risk; the released artifacts therefore comprise benchmark data and research code rather than a turnkey autonomous agent.}

\paragraph{Asset attribution and licensing.}
The paper builds on existing benchmarks, backbone models, and baseline implementations, all of which are credited in the main text and bibliography. {\color{black}The released repository preserves attribution and provides access notes for the licenses and terms governing TaskBench, ToolBench, and the backbone models or APIs.}

\begingroup
\color{black}
\makeatletter
\let\papercolor@floatboxreset\@floatboxreset
\def\@floatboxreset{\papercolor@floatboxreset\color{black}}
\makeatother
\captionsetup[table]{labelfont={color=black},textfont={color=black}}
\section{Additional Evaluation Diagnostics}
\label{app:additional_diagnostics}

{\color{black}This section complements the main experiments with diagnostics of benchmark validity, cycle semantics, component effects, structural accuracy, transfer, and robustness. Unless stated otherwise, the analyses use the same fixed test examples and evaluation definitions as the main experiments; they characterize TopoPlanner in this controlled setting rather than establish runtime correctness in an unconstrained open-world environment.}

\subsection{Benchmark Validity Beyond Topology-Guided Augmentation}
\label{app:benchmark_validity}

\paragraph{Motif-leakage audit.}
We group augmented examples by normalized request and exact ordered node--link structure before splitting. Table~\ref{tab:motif_leakage} reports no repeated normalized construction key between train and test. {\color{black}A few graph-only overlaps remain in HuggingFace and Multimedia because different requests can instantiate the same valid motif; none occur in Daily Life.}

\begin{table}[h]
\centering
\small
\setlength{\tabcolsep}{4pt}
\begin{tabular}{lrrr}
\toprule
Domain & Aug. test & Repeated key & Graph-only overlap \\
\midrule
HuggingFace & 300 & 0 & 21 \\
Multimedia & 220 & 0 & 13 \\
Daily Life & 160 & 0 & 0 \\
\bottomrule
\end{tabular}
\caption{{\color{black}Motif-leakage audit on augmented TaskBench domains.}}
\label{tab:motif_leakage}
\end{table}

\paragraph{Held-out topology.}
{\color{black}We evaluate held-out structure in one controlled setting: a merge topology absent from the corresponding training split. Table~\ref{tab:heldout_merge} reports TopoPlanner and its paired improvement over GTool on the same HuggingFace/Vicuna-7B examples, with confidence intervals obtained by paired bootstrap over test examples. This evidence is specific to the tested setting and does not establish generalization to every unseen workflow topology.}

\begin{table}[h]
\centering
\small
\begin{tabular}{lcc}
\toprule
Metric & TopoPlanner & Improvement over GTool (95\% CI) \\
\midrule
N-F1 & 0.787 & $+0.266$ $[0.240,0.292]$ \\
L-F1 & 0.491 & $+0.306$ $[0.273,0.339]$ \\
Exact match & 0.436 & $+0.352$ $[0.306,0.400]$ \\
Precedence accuracy & 0.591 & $+0.358$ $[0.321,0.397]$ \\
\bottomrule
\end{tabular}
\caption{{\color{black}Held-out merge-topology results on HuggingFace/Vicuna-7B.}}
\label{tab:heldout_merge}
\end{table}

\paragraph{Original non-augmented workflows.}
To test whether the observed gains are confined to topology-guided evaluation examples, Table~\ref{tab:original_workflows} evaluates original DAG workflows and the subset whose undirected workflow graph has positive cycle rank. The latter contains 47 of 837 HuggingFace DAGs and 12 of 565 Multimedia DAGs. These examples come from the source benchmarks rather than the augmented loop/merge construction. This diagnostic evaluates performance on original examples; by itself, it does not isolate the effect of using augmented examples during training.

\begin{table*}[h]
\centering
\scriptsize
\setlength{\tabcolsep}{4pt}
\begin{tabular}{llcc|cc}
\toprule
& & \multicolumn{2}{c|}{Original DAG (N-F1/L-F1)} & \multicolumn{2}{c}{Original topology-bearing subset (N-F1/L-F1)} \\
Domain & Backbone & GTool & TopoPlanner & GTool & TopoPlanner \\
\midrule
HuggingFace & Vicuna-7B & 48.80/10.78 & 84.38/48.00 & 53.76/17.13 & 79.60/49.50 \\
HuggingFace & CodeLlama-7B & 66.81/25.89 & 82.49/47.11 & 73.06/41.10 & 82.43/54.75 \\
Multimedia & Vicuna-7B & 57.79/27.29 & 82.31/49.01 & 48.28/15.11 & 70.22/40.22 \\
Multimedia & CodeLlama-7B & 76.85/43.52 & 84.45/50.00 & 72.28/35.44 & 83.23/51.64 \\
\bottomrule
\end{tabular}
\caption{{\color{black}Results on original workflows (\%).}}
\label{tab:original_workflows}
\end{table*}

\subsection{Executable Interpretation of Workflow Cycles}
\label{app:cycle_semantics}

The \textit{Cycle} category groups loop, merge, and loop-merge motifs because each induces a closed cycle-level region in the lifted cellular representation. A merge is represented as a cycle-level cell because its multiple dependency paths bound a closed region; this structural notion does not imply repeated execution. For motifs containing a directed loop, the loop represents bounded feedback, revision, reuse, or finite transformation at the workflow-graph level, rather than mutually blocking preconditions or an instruction to execute indefinitely. We audit all 2,266 loop-bearing requests using non-exclusive semantic cue categories. {\color{black}Requests unmatched by the cue lexicon are manually assigned to finite transformation only if they specify a finite tool-call sequence.}

\begin{table*}[h]
\centering
\small
\setlength{\tabcolsep}{4pt}
\begin{tabular}{lrrrrrr}
\toprule
Subset & \#Req. & Condition/stop & Verify/check & Refine/correct & Finite cyclic & Bounded/finite \\
\midrule
All loop-bearing & 2,266 & 27.89 & 52.07 & 81.16 & 1.99 (45) & 100.00 \\
Directed loop & 1,134 & 12.43 & 43.03 & 81.66 & 3.00 (34) & 100.00 \\
Loop-merge & 1,132 & 43.37 & 61.13 & 80.65 & 0.97 (11) & 100.00 \\
\bottomrule
\end{tabular}
\caption{{\color{black}Bounded/finite semantics audit for loop-bearing requests (\%; cues may overlap).}}
\label{tab:cycle_semantics}
\end{table*}

{\color{black}We also inspect raw generated tool sequences with a diagnostic guard of $\max(8,2\times\text{gold tool count})$. ``Repeat'' flags any repeated tool call, ``over-step'' flags a sequence exceeding the guard, and ``safe stop'' denotes that neither trigger occurs. This offline diagnostic does not replace executor-level stopping.}

\begin{table}[h]
\centering
\small
\begin{tabular}{lrrrrr}
\toprule
Method & Safe stop & Repeat & Over-step & C-M safe & Cycle safe \\
\midrule
GTool & 90.15 & 6.08 & 8.43 & 78.0 & 82.0 \\
TopoPlanner & 99.11 & 0.79 & 0.10 & 96.0 & 99.0 \\
\bottomrule
\end{tabular}
\caption{{\color{black}Offline guarded-stop audit (\%; repeat/over-step $\downarrow$).}}
\label{tab:guarded_stop}
\end{table}

\subsection{Component Diagnostics}
\label{app:component_diagnostics}

Table~\ref{tab:component_diagnostics} extends the main ablation with inference and representation diagnostics. The hypergraph variant tests whether arbitrary higher-order grouping alone explains the gain, while the no-consistency and no-closure variants isolate the two components of the retrieval criterion. The former rows alter the trained architecture or objective, whereas ``no Eq.~9 consistency'' and ``no boundary closure'' are inference-time interventions; the two groups should therefore not be interpreted as directly interchangeable ablations. {\color{black}Across both backbones and datasets, the full model exceeds every tested variant on N-F1, L-F1, and Accuracy, although the effect sizes vary by setting.}

\begin{table*}[h]
\centering
\scriptsize
\setlength{\tabcolsep}{4pt}
\begin{tabular}{llcc}
\toprule
Backbone & Variant & HuggingFace N/L/Acc & Multimedia N/L/Acc \\
\midrule
Vicuna-7B & w/o 2-cell & 75.58/46.61/30.88 & 73.29/50.13/34.15 \\
& w/o retrieval & 76.07/47.88/32.21 & 77.97/55.11/35.37 \\
& w/o $\mathcal{L}_{CI}$ & 80.69/53.31/38.68 & 81.47/58.29/40.98 \\
& no Eq.~9 consistency & 77.39/49.09/33.68 & 80.95/57.03/41.71 \\
& no boundary closure & 75.66/47.07/32.50 & 81.41/57.71/43.17 \\
& hypergraph variant & 79.47/51.66/37.65 & 82.38/59.73/44.63 \\
& TopoPlanner & \textbf{82.23/55.72/42.35} & \textbf{86.80/66.24/54.76} \\
\midrule
CodeLlama-7B & w/o 2-cell & 78.95/51.87/35.29 & 75.90/56.35/43.54 \\
& w/o retrieval & 79.75/56.96/35.59 & 83.89/62.92/49.27 \\
& w/o $\mathcal{L}_{CI}$ & 80.80/53.83/38.82 & 87.12/67.61/53.54 \\
& no Eq.~9 consistency & 79.44/53.19/38.24 & 86.24/64.34/53.41 \\
& no boundary closure & 78.74/52.00/35.59 & 85.59/66.14/55.37 \\
& hypergraph variant & 77.35/50.24/31.32 & 84.20/62.79/49.15 \\
& TopoPlanner & \textbf{83.43/59.62/47.35} & \textbf{89.08/69.59/62.20} \\
\bottomrule
\end{tabular}
\caption{{\color{black}Extended component diagnostics (\%); training and inference variants are not directly comparable.}}
\label{tab:component_diagnostics}
\end{table*}

\subsection{Structure-Specific Evaluation}
\label{app:structure_diagnostics}

\paragraph{Merge and reuse dependencies.}
Merge edges point to nodes with in-degree greater than one, reuse edges leave nodes with out-degree greater than one, and M/R recall is computed on their union. {\color{black}The four-domain 7B subgroup contains 804 examples and 2,325 gold merge/reuse edges for each backbone.} {\color{black}Merge and reuse are not independent cell types; this diagnostic tests regional dependency preservation.}

\begin{table}[h]
\centering
\small
\begin{tabular}{llrrrr}
\toprule
Method & Backbone & $n$ & N-F1 & L-F1 & M/R Rec. \\
\midrule
GTool & Vicuna-7B & 804 & 51.18 & 22.06 & 19.04 \\
TopoPlanner & Vicuna-7B & 804 & 76.16 & 46.92 & 39.72 \\
GTool & CodeLlama-7B & 804 & 64.56 & 35.27 & 29.06 \\
TopoPlanner & CodeLlama-7B & 804 & 78.94 & 49.91 & 41.01 \\
\bottomrule
\end{tabular}
\caption{{\color{black}Merge/reuse subgroup results (\%).}}
\label{tab:merge_reuse}
\end{table}

\paragraph{Execution order.}
For each predicted tool sequence, we check each gold dependency $u\rightarrow v$. Pair measures macro-averaged correctness over gold dependency pairs, with missing endpoints counted as errors. {\color{black}Dep. is an example-level all-or-nothing score requiring every gold dependency's endpoints to be present and ordered correctly.}

\begin{table*}[h]
\centering
\small
\setlength{\tabcolsep}{4pt}
\begin{tabular}{llcc|cc|cc}
\toprule
& & \multicolumn{2}{c|}{HuggingFace} & \multicolumn{2}{c|}{Multimedia} & \multicolumn{2}{c}{Daily Life} \\
Method & Backbone & Dep. & Pair & Dep. & Pair & Dep. & Pair \\
\midrule
GTool & Vicuna-7B & 9.71 & 24.18 & 8.29 & 23.69 & 39.86 & 57.32 \\
TopoPlanner & Vicuna-7B & 27.94 & 51.74 & 40.37 & 64.29 & 60.14 & 79.50 \\
GTool & CodeLlama-7B & 25.74 & 44.91 & 32.56 & 49.73 & 48.11 & 62.30 \\
TopoPlanner & CodeLlama-7B & 33.24 & 58.64 & 48.78 & 71.80 & 63.57 & 83.28 \\
\bottomrule
\end{tabular}
\caption{{\color{black}Execution-order results (\%).}}
\label{tab:execution_order}
\end{table*}

\subsection{Paired Statistical Significance}
\label{app:paired_significance}

We pair TopoPlanner and GTool predictions on the same fixed test instances. N-F1, L-F1, and Accuracy differences are evaluated by example-level paired bootstrap with 10,000 resamples. Table~\ref{tab:paired_significance} reports representative paired differences and 95\% confidence intervals. {\color{black}All listed two-sided paired-bootstrap tests have $p<0.0002$.} {\color{black}These intervals quantify example-level uncertainty for fixed checkpoints, not variation across training runs.}

\begin{table}[H]
\centering
\small
\setlength{\tabcolsep}{4pt}
\begin{tabular}{llccc}
\toprule
Domain & Backbone & N-F1 difference & L-F1 difference & Accuracy difference \\
\midrule
HuggingFace & Vicuna-7B & $+0.298$ $[0.276,0.319]$ & $+0.359$ $[0.331,0.387]$ & $+0.349$ $[0.312,0.387]$ \\
HuggingFace & CodeLlama-7B & $+0.149$ $[0.129,0.169]$ & $+0.201$ $[0.172,0.229]$ & $+0.228$ $[0.188,0.269]$ \\
Multimedia & Vicuna-7B & $+0.392$ $[0.373,0.411]$ & $+0.466$ $[0.440,0.492]$ & $+0.482$ $[0.446,0.517]$ \\
Multimedia & CodeLlama-13B & $+0.162$ $[0.144,0.182]$ & $+0.185$ $[0.162,0.208]$ & $+0.283$ $[0.250,0.317]$ \\
Daily Life & Vicuna-13B & $+0.166$ $[0.144,0.189]$ & $+0.166$ $[0.142,0.189]$ & $+0.301$ $[0.261,0.340]$ \\
Daily Life & CodeLlama-13B & $+0.144$ $[0.124,0.165]$ & $+0.139$ $[0.116,0.162]$ & $+0.261$ $[0.222,0.301]$ \\
\bottomrule
\end{tabular}
\caption{{\color{black}Paired TopoPlanner--GTool differences with 95\% bootstrap CIs (scores in $[0,1]$).}}
\label{tab:paired_significance}
\end{table}

\subsection{Cross-Domain Transfer and Incomplete-Graph Robustness}
\label{app:robustness}

\paragraph{Graph-grounded transfer.}
The source-domain checkpoint is evaluated in a different target domain without target-domain parameter updates. Target-domain tool descriptions and the target dependency graph remain available at inference. This is graph-grounded transfer to unseen-but-described tools and target graphs, not unconstrained discovery of unregistered tools.

\begin{table}[H]
\centering
\scriptsize
\setlength{\tabcolsep}{3pt}
\begin{tabular}{lllccc|ccc|ccc}
\toprule
Backbone & Source & Target & \multicolumn{3}{c|}{Direct} & \multicolumn{3}{c|}{GTool} & \multicolumn{3}{c}{TopoPlanner} \\
& & & N & L & Acc & N & L & Acc & N & L & Acc \\
\midrule
Vicuna-7B & Daily Life & Multimedia & .2897 & .0038 & .0171 & .2288 & .0106 & .0122 & .5120 & .1820 & .0646 \\
Vicuna-7B & HuggingFace & Daily Life & .3200 & .1274 & .0756 & .3195 & .1251 & .0739 & .7436 & .4012 & .2474 \\
Vicuna-7B & Multimedia & HuggingFace & .4218 & .1151 & .0647 & .3576 & .0766 & .0294 & .4571 & .1736 & .0676 \\
CodeLlama-7B & Daily Life & Multimedia & .4963 & .0527 & .1122 & .3893 & .1266 & .0268 & .5717 & .2905 & .1305 \\
CodeLlama-7B & HuggingFace & Daily Life & .6844 & .4362 & .3265 & .6759 & .4214 & .3247 & .8955 & .6407 & .6065 \\
CodeLlama-7B & Multimedia & HuggingFace & .5619 & .2888 & .0912 & .5821 & .2445 & .1059 & .6722 & .3567 & .1868 \\
\bottomrule
\end{tabular}
\caption{{\color{black}Cross-domain transfer without target-domain updates (scores in $[0,1]$).}}
\label{tab:cross_domain_transfer}
\end{table}

\paragraph{Incomplete dependency graphs.}
We keep the gold workflows unchanged, remove a controlled fraction of dependency edges from the input graph before training and inference, and recompute edge and 2-cell features from the remaining graph. Table~\ref{tab:incomplete_graph} shows that node selection is comparatively stable, whereas link recovery degrades more strongly because it depends directly on the available structural evidence. {\color{black}This diagnostic covers controlled missing edges, not general corruption.}

\begin{table}[H]
\centering
\small
\begin{tabular}{lrrrr}
\toprule
Dataset & Missing edges & N-F1 & L-F1 & Acc \\
\midrule
HuggingFace & 0\% & 83.43 & 59.62 & 47.35 \\
HuggingFace & 10\% & 82.01 & 53.50 & 43.68 \\
HuggingFace & 20\% & 80.77 & 44.17 & 40.00 \\
HuggingFace & 50\% & 80.49 & 31.61 & 38.09 \\
Multimedia & 0\% & 89.08 & 69.59 & 62.20 \\
Multimedia & 10\% & 87.52 & 60.97 & 58.78 \\
Multimedia & 20\% & 87.79 & 58.96 & 59.63 \\
Multimedia & 50\% & 87.69 & 43.42 & 58.78 \\
\bottomrule
\end{tabular}
\caption{{\color{black}Robustness to missing dependency edges (\%).}}
\label{tab:incomplete_graph}
\end{table}

\subsection{Retrieval Scale and Workflow Length}
\label{app:scale_length}

Retrieval scores cells in the available tool graph rather than enumerating every path or subgraph. All tool nodes and 2-cells are retained; only the dense Daily Life edge set uses a deterministic 512-edge prefilter. {\color{black}Table~\ref{tab:retrieval_pool} reports timing without LLM generation.}

\begin{table}[H]
\centering
\small
\begin{tabular}{lrrr}
\toprule
Dataset & Pool (nodes/edges/2-cells) & Edge candidates & CPU ms \\
\midrule
HuggingFace & 23/225/54 & 225 & 23.65 \\
Multimedia & 40/449/85 & 449 & 33.11 \\
Daily Life & 40/1560/146 & 512 & 39.45 \\
ToolBench & 245/366/19 & 366 & 22.89 \\
\bottomrule
\end{tabular}
\caption{{\color{black}Retrieval-pool size and mean CPU time over 50 fixed queries.}}
\label{tab:retrieval_pool}
\end{table}

Table~\ref{tab:length_analysis} groups examples by gold tool-call length and compares TopoPlanner with GTool on the same examples. The gains generally increase for medium and long workflows, while the lower length MAE indicates that the improvement is not explained only by producing longer sequences.

\begin{table}[H]
\centering
\small
\setlength{\tabcolsep}{4pt}
\begin{tabular}{llrrrrr}
\toprule
Dataset & Gold length & \#Ex. & $\Delta$N-F1 & $\Delta$L-F1 & $\Delta$Acc & Length MAE (GTool $\rightarrow$ Ours) \\
\midrule
HuggingFace & 1--2 & 31 & +5.64 & +11.83 & +12.10 & 0.32 $\rightarrow$ 0.19 \\
& 3--4 & 484 & +14.94 & +18.59 & +22.37 & 0.67 $\rightarrow$ 0.27 \\
& $\geq 5$ & 165 & +24.62 & +26.03 & +18.34 & 1.52 $\rightarrow$ 0.96 \\
Multimedia & 1--2 & 45 & +6.08 & +13.33 & +15.55 & 0.20 $\rightarrow$ 0.07 \\
& 3--4 & 513 & +23.90 & +27.98 & +38.84 & 0.94 $\rightarrow$ 0.31 \\
& $\geq 5$ & 262 & +27.50 & +30.88 & +33.31 & 1.68 $\rightarrow$ 0.65 \\
Daily Life & 1--2 & 25 & +1.38 & +2.22 & +4.00 & 0.12 $\rightarrow$ 0.09 \\
& 3--4 & 416 & +15.44 & +16.87 & +32.63 & 0.71 $\rightarrow$ 0.07 \\
& $\geq 5$ & 141 & +27.14 & +28.07 & +40.07 & 1.99 $\rightarrow$ 0.17 \\
\bottomrule
\end{tabular}
\caption{{\color{black}TopoPlanner--GTool differences by gold tool-call length (percentage points).}}
\label{tab:length_analysis}
\end{table}

\endgroup
\section{{\color{black}Workflow Examples and Structural Interpretation}}
\label{app:qualitative_analysis}

{\color{black}The following examples illustrate how loop, merge, and loop-merge workflows are represented in the lifted complex. They explain the structural information available to TopoPlanner; because per-example baseline predictions are not shown, they are not presented as qualitative evidence of baseline failure.}

\subsection{Case Study 1: HuggingFace}
\textbf{Example ID:} \texttt{paper\_cycle\_huggingface\_0005}. \\
\textbf{Request summary:} extract information from a historical document, highlight key entities and dates, convert the result to speech, enhance the audio, transcribe it back, and then revisit the original document question answering step using the refined transcript.

\paragraph{Topology.}
This workflow forms a {\color{black}loop-merge} pattern:
\texttt{Document QA \(\rightarrow\) Token Classification \(\rightarrow\) Text-to-Speech \(\rightarrow\) Audio-to-Audio \(\rightarrow\) ASR \(\rightarrow\) Document QA},
with an additional direct edge \texttt{Document QA \(\rightarrow\) Text-to-Speech}. The direct edge matters because the speech synthesis stage must inherit both the initial extraction and the entity-focused refinement, while the ASR output closes a feedback loop for iterative correction.

\paragraph{{\color{black}Structural interpretation.}}
{\color{black}The lifted representation places the direct merge path and the feedback path on the boundary of the same cycle-level region. Boundary-closed retrieval therefore exposes both relations to the planner rather than representing only a single linear path.}

\subsection{Case Study 2: Multimedia}
\textbf{Example ID:} \texttt{paper\_cycle\_multimedia\_0002}. \\
\textbf{Request summary:} transcribe a raw video interview, summarize it, generate relevant topics, rewrite the narration, denoise the audio, synchronize the new voiceover with the video, and refine the narration again if synchronization is unsatisfactory.

\paragraph{Topology.}
The workflow is:
\texttt{Video-to-Text \(\rightarrow\) Text Summarizer \(\rightarrow\) Topic Generator \(\rightarrow\) Voice Changer \(\rightarrow\) Audio Noise Reduction \(\rightarrow\) Video Synchronization \(\rightarrow\) Video-to-Text},
with an additional direct edge \texttt{Video-to-Text \(\rightarrow\) Topic Generator}. This creates a multi-stage feedback loop in which the synchronization result may trigger another transcription-and-revision pass.

\paragraph{{\color{black}Structural interpretation.}}
{\color{black}The region contains both the shortcut into \texttt{Topic Generator} and the closing edge back to \texttt{Video-to-Text}, making the dual-source topic dependency and synchronization feedback jointly available during cellular reasoning.}

\subsection{Case Study 3: Daily Life}
\textbf{Example ID:} \texttt{paper\_cycle\_dailylife\_0005}. \\
\textbf{Request summary:} share a conference poster, print it for a manager, book a flight, pay the registration bill, search for attendee reviews, and if necessary share a cancellation notice and rebook a later flight.

\paragraph{Topology.}
The workflow is a loop-merge pattern:
\texttt{share\_by\_social\_network \(\rightarrow\) print\_document \(\rightarrow\) book\_flight \(\rightarrow\) daily\_bill\_payment \(\rightarrow\) search\_by\_engine \(\rightarrow\) share\_by\_social\_network},
with an additional direct edge \texttt{share\_by\_social\_network \(\rightarrow\) book\_flight}. {\color{black}The request therefore couples communication, logistics, payment, and post-hoc revision in a closed rather than one-way workflow.}

\paragraph{{\color{black}Structural interpretation.}}
{\color{black}Boundary closure keeps the direct booking dependency together with the late feedback edge that reopens communication and trip planning, representing the request as a closed region rather than a one-way chain.}

\subsection{Case Study 4: ToolBench}
\textbf{Example ID:} \texttt{toolbench\_circle\_0005}. \\
\textbf{Request summary:} generate a face image, consult a zodiac-sign prediction, produce a flirty line conditioned on that prediction, and then adjust the generated face based on the tone of the final line.

\paragraph{Topology.}
This ToolBench example forms a compact {\color{black}loop-merge} structure:
\texttt{generate a face \(\rightarrow\) zodiac prediction \(\rightarrow\) generate flirty line \(\rightarrow\) generate a face},
with an additional direct edge \texttt{generate a face \(\rightarrow\) generate flirty line}. The direct edge expresses that the final line depends not only on zodiac information but also on the face-generation context, while the feedback edge turns the line into a refinement signal for the image generator.

\paragraph{{\color{black}Structural interpretation.}}
{\color{black}Here the cycle-level cell couples the direct face-to-line dependency with the feedback edge from the line to image generation. This makes the merge and refinement relations jointly available even in ToolBench's larger tool inventory.}
\section{Additional Related Work}
\label{app:related}
\paragraph{Topological Deep Learning.} 

Topological deep learning (TDL) builds upon early advances in Topological Signal Processing (TSP)~(\cite{barbarossa2020topological, sardellitti2022cell}), which emphasized that many real-world systems exhibit interactions that go beyond pairwise relations and are more naturally described by high-dimensional topological structures. This insight has motivated the generalization of classical graph-theoretic tools from pairwise graphs to richer topological domains, such as simplicial complexes (SCs), which model interactions among fixed-size groups, and cell complexes (CWs), which further generalize this representation to arbitrary polyhedral cells.

Along this line, extensions of the Weisfeiler--Lehman (WL) test to SCs and CWs~(\cite{bodnar2021weisfeiler, DBLP:conf/nips/BodnarFOWLMB21}) have provided theoretical foundations for distinguishing and propagating high-dimensional structural information, thereby enabling principled message passing beyond nodes and edges. {\color{black}Building on these ideas, convolutional architectures operate on SCs and CWs~(\cite{ebli2020simplicial, yang2022effsimpl, hajij2020cell, yang2023convolutional, hajij2022}), while attentional variants model higher-dimensional neighborhoods~(\cite{goh2022simplicial, giusti2023cell}).}

To unify these heterogeneous developments, the combinatorial complex (CC) framework~(\cite{hajij2023topological}) was introduced as a general abstraction that subsumes SCs, CWs, and hypergraphs within a single message-passing formalism, allowing a systematic treatment of multi-dimensional relational structures. In parallel, Sheaf Neural Networks (SNNs)~(\cite{HGh19, hansen2020sheaf, sheaf2022, battiloro2024tangent, barbero2022sheaf,wei2025persistent}) have explored a complementary perspective by modeling local consistency constraints via sheaf theory, demonstrating strong performance in settings with heterophily and structured local interactions.

Recently, topological perspectives have been applied to analyze reasoning in large language models. \cite{DBLP:journals/corr/abs-2506-05744} introduce \emph{reasoning graphs} by clustering hidden-state representations across reasoning steps and analyze their structural properties---such as cyclicity, diameter, and small-world characteristics---on mathematical reasoning benchmarks. These developments suggest that topological structure is increasingly relevant to LLM reasoning, but direct topology-aware planning over tool workflows remains underexplored, which is the gap our work addresses.

\section{Computational Complexity Analysis}
\label{app:complexity}

{\color{black}This section characterizes the computational cost of TopoPlanner.} The total computational overhead can be decomposed into four stages aligned with the model design: cellular workflow lifting, cosheaf-consistent cellular retrieval, multi-dimensional structural reasoning, and LLM tool planning.

\subsection{Theoretical Analysis}

\paragraph{1. Topological Lifting.}
{\color{black}The construction extracts a fundamental cycle basis from the underlying undirected graph $G_u=(V,E)$. For a connected graph, we first construct a spanning tree $\mathcal{T}$ in $O(|V|+|E|)$ time. Each of the $\beta_1=|E|-|V|+1$ non-tree edges induces one fundamental cycle; standard cycle-basis algorithms compute the basis in $O(|V|\cdot|E|)$ time.}
\begin{itemize}
    \item \textbf{Efficiency:} This lifting can be precomputed for each dataset/tool graph and cached together with the graph features. During training and inference, TopoPlanner reads the precomputed cellular structure and only applies request-dependent retrieval and encoding.
\end{itemize}

\paragraph{2. Cosheaf-Consistent Cellular Retrieval.}
For each request, TopoPlanner scores cells by request relevance and cosheaf-induced boundary compatibility. The practical implementation first applies a cheap edge prefilter, then performs closure-preserving greedy selection. {\color{black}Its cost is linear in the number of retained candidate cells after prefiltering; Table~\ref{tab:retrieval_pool} reports the measured retrieval time separately from LLM decoding.}

\paragraph{3. Multi-Dimensional Structural Reasoning.}
Our encoder employs a message-passing neural network over the retrieved cell complex. Let $\mathcal{X}^{*}_q$ denote the retrieved subcomplex and $\mathcal{E}_{adj}$ denote the boundary and coface adjacency connections in the lifted complex.
\begin{itemize}
    \item {\color{black}One reasoning layer costs $O(|\mathcal{X}^{*}_q| \cdot d^2 + |\mathcal{E}_{adj}| \cdot d)$ for hidden dimension $d$.}
    \item {\color{black}\textbf{Linear scaling in retained incidences:} Unlike global Transformer attention, which scales as $O(|\mathcal{X}^{*}_q|^2)$, message passing is confined to boundary and coface relations. Its cost is therefore linear in the retained incidence connections; for the dense Daily Life edge pool, retrieval first applies the 512-edge prefilter in Table~\ref{tab:retrieval_pool}.}
\end{itemize}

\paragraph{4. LLM Tool Planning Overhead.}
{\color{black}The local open-source variant decodes autoregressively with a frozen planner LLM and one injected graph token. The API variant uses topology to rerank proposals.}
\begin{itemize}
    \item {\color{black}LLM latency dominates, with approximate per-token complexity $O(T \cdot d_{model}^2)$.}
    \item {\color{black}\textbf{Measured retrieval cost:} Table~\ref{tab:retrieval_pool} reports mean CPU times of 22.89--39.45 ms over 50 fixed queries. These measurements exclude LLM generation and therefore do not establish end-to-end latency overhead.}
\end{itemize}

\subsection{Scalability Discussion}
{\color{black}Enumerating all simple cycles can be exponential, especially in dense graphs. TopoPlanner instead constructs one fundamental cycle per non-tree edge, so the untruncated number of 2-cells equals the cycle rank $\beta_1=|E|-|V|+1$ for a connected graph. This remains linear in graph size even for the dense Daily Life graph; its request-time retrieval additionally applies the deterministic 512-edge prefilter reported in Table~\ref{tab:retrieval_pool}.}

\end{document}